\documentclass[11pt]{article} 
\usepackage{fullpage}
\usepackage{graphics} 
\usepackage{epsfig}
\usepackage{amsmath} 
\usepackage{amssymb} 
\usepackage{amsthm}
\usepackage{tabularx,booktabs}
\usepackage{mathrsfs}
\usepackage{selectp}
\usepackage[normalem]{ulem}
\usepackage{authblk}
\usepackage{subcaption}
\usepackage{natbib}
\usepackage{Shorthands}
\usepackage{algorithm}
\usepackage{algpseudocode}
\usepackage{wrapfig}
\usepackage{scalerel,stackengine}
\stackMath
\newcommand\reallywidehat[1]{%
\savestack{\tmpbox}{\stretchto{%
  \scaleto{%
    \scalerel*[\widthof{\ensuremath{#1}}]{\kern-.6pt\bigwedge\kern-.6pt}%
    {\rule[-\textheight/2]{1ex}{\textheight}}%WIDTH-LIMITED BIG WEDGE
  }{\textheight}% 
}{0.5ex}}%
\stackon[1pt]{#1}{\tmpbox}%
}

\usepackage{hyperref}
\hypersetup{       
    unicode=false,          
    pdftoolbar=true,       
    pdfmenubar=true,      
    pdffitwindow=false,     
    pdfstartview={FitH},   
    pdfnewwindow=true,      
    colorlinks=true,      
    linkcolor=blue,          % color of internal links (change box color with linkbordercolor)
    citecolor=blue,       
    filecolor=magenta,      
    urlcolor=blue,           
    breaklinks=true
}

\usepackage[toc,page,header]{appendix}
\usepackage{minitoc}

\title{\LARGE \bf
Learning Stabilizing Policies via an Unstable \\
Subspace Representation
}

\author[1]{Leonardo F. Toso\footnote{Email addresses: \texttt{\{lt2879,james.anderson\}@columbia.edu}, \texttt{yelintao93@hust.edu.cn}.}}
\author[2]{Lintao Ye}
\author[1]{James Anderson}
\affil[1]{Department of Electrical Engineering, Columbia University}
\affil[2]{School of Artificial Intelligence and
Automation, HUST}
\date{}
\begin{document}

\doparttoc 
\faketableofcontents

\maketitle
\allowdisplaybreaks
\begin{abstract}
We study the problem of learning to stabilize (LTS) a linear time-invariant (LTI) system. Policy gradient (PG) methods for control assume access to an initial stabilizing policy. However, designing such a policy for an \emph{unknown} system is one of the most fundamental problems in control, and it may be as hard as learning the optimal policy itself. Existing work on the LTS problem requires large data as it scales quadratically with the ambient dimension. We propose a two-phase approach that first learns the \emph{left unstable subspace} of the system and then solves a series of discounted linear quadratic regulator (LQR) problems on the learned unstable subspace, targeting to stabilize only the system's unstable dynamics and reduce the effective dimension of the control space. We provide non-asymptotic guarantees for both phases and demonstrate that operating on the unstable subspace reduces sample complexity. In particular, when the number of unstable modes is much smaller than the state dimension, our analysis reveals that LTS on the unstable subspace substantially speeds up the stabilization process. Numerical experiments are provided to support this sample complexity reduction achieved by our approach.
\end{abstract}

\section{Introduction}

In contrast to traditional model-based control methods, model-free, policy gradient (PG) approaches offer two substantial advantages: (i) they are simple to implement without requiring knowledge of the underlying system dynamics, and (ii) they adapt readily to new tasks with minimal parameter tuning. These methods have been widely used to solve reinforcement learning (RL) tasks in unknown environments \citep{sutton1999policy}, with recent work establishing strong optimality guarantees \citep{agarwal2021theory}. As a result, there has been much interest in applying PG methods to optimal control, see the excellent review by \cite{hu2023toward} for an overview. Problems of particular relevance to this work includes the linear quadratic regulator (LQR) problem in the offline setting \citep{fazel2018global,malik2019derivative,gravell2020learning,mohammadi2021convergence}, online setting \citep{cassel2021online}, multi-task setting \citep{wang202model,toso2024meta,toso2024async, zhan2025coreset}, and networked setting \citep{mitra2024towards}. A crucial milestone was achieved in \cite{fazel2018global}, which showed that the LQR problem exhibits a benign optimization landscape, enabling global convergence of PG methods (with linear rate as shown in \cite{mohammadi2020linear}).

There is however a major obstacle encountered when applying PG methods to control: it is typically assumed that one has access to an initial stabilizing policy. For one of the most fundamental problems in control, that of finding a stabilizing policy for an unknown system, such an assumption precludes the use of PG methods. In particular, learning to stabilize (LTS) a linear system can be as hard as learning the optimal policy itself \citep{tsiamis2022learning,zeng2023hardness}.

Several solutions to the LTS problem have been proposed, c.f., \citep{lale2020explore, lamperski2020computing, chen2021black,perdomo2021stabilizing,hu2022sample, zhao2024convergence}. Two notable existing approaches that this work builds on are: 
discounted methods \citep{lamperski2020computing,perdomo2021stabilizing, zhao2024convergence} and unstable subspace learning \citep{hu2022sample,zhang2024learning,werner2025system}. In the first, PG solves discounted LQR problems with a carefully selected sequence of increasing discount factors. Since the policy gradients are estimated from data (i.e., system trajectories), this approach typically suffers from a high sample complexity as it scales quadratically with the ambient problem dimension \citep{zhao2024convergence}. 

On the other hand, since a stabilizing policy only needs to address the system’s unstable dynamics, focusing on stabilizing just the unstable modes reduces the effective dimensionality of the control space and consequently, the sample complexity, as shown in \cite{hu2022sample} for the noiseless setting and in \cite{zhang2024learning} for the stochastic setting. However, these works rely on identifying the full unstable dynamics to construct a stabilizing policy on top of the identified model, therefore being model-based and highly sensitive to the model's estimation accuracy. Furthermore, the analyses in \cite{hu2022sample,zhang2024learning} are restricted to diagonalizable systems.

\begin{figure}
  \centering
  \includegraphics[width=1\textwidth]{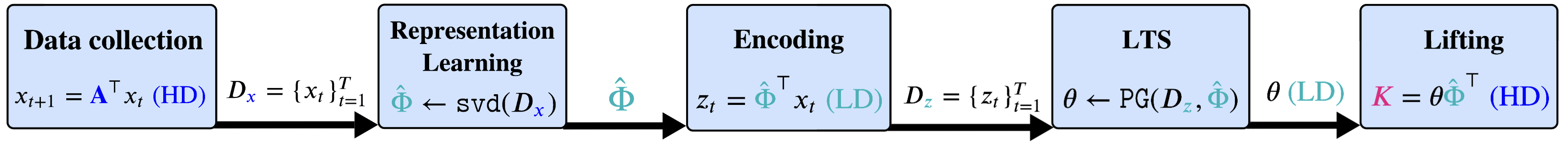}
  \caption{Workflow for learning to stabilize (LTS) a high-dimensional (HD) discrete-time LTI system on its low-dimensional (LD) unstable subspace.}
  \label{fig:LTS_pipeline}
\end{figure}

In contrast, our approach accommodates non-diagonalizable system and combines the strengths of both perspectives: we avoid explicitly identifying the system's unstable dynamics while using policy gradient to stabilize only the unstable modes. In particular, we solve a sequence of discounted LQR problems by performing policy gradient updates on the left unstable subspace of the system (see Figure~\ref{fig:LTS_pipeline}). Moreover, our work addresses the following questions:
\vspace{0.35cm}

\noindent $\bullet$ \emph{To what extent can we guarantee the stability of a high-dimensional system by performing a discount-factor annealing method on its low-dimensional unstable subspace?}
\vspace{0.1cm}

\noindent $\bullet$ \emph{How does this approach reduce the sample complexity of learning a stabilizing controller?}
\vspace{0.1cm}

\noindent $\bullet$ \emph{What is the sample complexity of estimating the representation of the left unstable subspace?}

\subsection{Contributions}

\noindent $\bullet$ \textbf{Sample complexity reduction:} 
    By operating on the unstable subspace, namely, the subspace associated with the system's $\ell \in \mathbb{N}$ unstable modes, we aim to stabilize only the portion of the system that requires stabilization rather than the full $\dx$-dimensional state space. We demonstrate the reduction in the sample complexity of finding a stabilizing policy from $\tilde{\mathcal{O}}(\textcolor{red}{\dx^2} \du)$ \citep{zhao2024convergence} to $\tilde{\mathcal{O}}(\textcolor{blue}{\ell^2} \du)$ (Theorem~\ref{theorem:main-result}), with $\du$ being the number of inputs, which is significant when the number of unstable modes is much smaller than the state dimension, i.e., $\ell \ll \dx$.
\vspace{0.1cm}

\noindent $\bullet$ \textbf{Learning the left unstable subspace:} We demonstrate that operating on the \emph{left} unstable subspace allows for controlling the closed-loop spectral radius in terms of the accuracy of the learned representation. We also provide finite-sample guarantees for learning this representation by sampling data from an adjoint system (Theorem~\ref{theorem:learning the left unstable subspace mb}). Therefore, the closed-loop spectral radius error decreases as more data is collected. 
This contrasts with prior work of \cite{hu2022sample, zhang2024learning}, which recovers a basis of the right unstable subspace. Their error bounds depend on a ``coupling term" that arises from decomposing the system's dynamics into stable and unstable components and inevitably incurs a bias that is significant for non-symmetric system matrices.
\vspace{0.1cm}

\noindent $\bullet$ \textbf{Non-diagonalizable matrices:} Our results accommodate non-diagonalizable systems. In contrast to \cite{hu2022sample, zhang2024learning}, which restrict the analysis to diagonalizable systems, we leverage the Jordan form decomposition and establish that the left unstable subspace representation can be learned with a finite amount of samples (Lemma \ref{lemma: lower bound on sigma_l} and Theorem \ref{theorem:learning the left unstable subspace mb}). That is in contrast to \cite{zhang2024learning}, where the sample complexity scales inversely with the spectral gap between the unstable modes; this dependence is problematic when the system is non-diagonalizable, as the gap goes to zero and data grows prohibitively large. We prove that it should not be the case.

\subsection{Related Work}

\noindent $\bullet$ \textbf{Learning to stabilize with identified model:} A natural idea to find a stabilizing controller for an unknown system is first to identify the system's model from data and then synthesize a controller on top of it.  \cite{chen2021black} show that the sample complexity scales exponentially with $\dx$ when learning to stabilize from a single trajectory. However, such scaling is undesirable when $\dx$ is large. To overcome this, \cite{hu2022sample} demonstrate that a stabilizing policy can be learned by only identifying the unstable modes of the system, which leads to a much more benign sample complexity that scales with the number of unstable modes $\ell \ll \dx$.
In contrast to \cite{hu2022sample}, this work does not require identifying the unstable dynamics; namely, we identify a basis (or representation) of the left unstable subspace of the system. Moreover, our approach accommodates non-diagonalizable matrices, which is not the case in \cite{hu2022sample,zhang2024learning}. 
\vspace{0.1cm}

\noindent $\bullet$ \textbf{Learning to stabilize with policy gradient:} An alternative approach is to learn a stabilizing controller \emph{without} performing system identification. Recent work \cite{lamperski2020computing,perdomo2021stabilizing}  showed that 
a reformulation of the LQR problem that involves introducing an additional degree of freedom--a ``damping factor'', $\gamma\in (0,1]$, leads to an intuitive, iterative approach for constructing a stabilizing policy. Initially, setting $\gamma$ sufficiently small, the trivial zero policy stabilizes the underlying damped system. PG methods solve the damped LQR problem and produce an initial stabilizing policy for the subsequent discounted LQR problem. Once a stabilizing controller is obtained, $\gamma$ is incrementally increased, and the process is repeated as $\gamma$ goes to one. 

\cite{zhao2024convergence} provide an explicit update rule for $\gamma$, which allows for characterizing the sample complexity of LTS with discounted PG. In particular, it scales as $\mathcal{O}(\dx^2 \du)$ and becomes prohibitively large for high-dimensional systems where $\dx$ is large. In this work, we only focus on stabilizing the system's unstable modes which reduces the sample complexity to $\mathcal{O}(\ell^2 \du)$. We emphasize that \cite{werner2025system} consider policy optimization on the unstable subspace to learn a stabilizing policy; however, they do not provide finite-sample guarantees for either the unstable subspace representation learning or the resulting stabilizing policy.
\vspace{0.1cm}

\noindent $\bullet$ \textbf{Representation learning for control:} We also stress the difference between the control policy representation considered in this work and the low-rank representation of the system model in \cite{zhang2024sampleefficient, lee2024regret}. The low-rank representation of the system model captures the important features to be identified and potentially shared across multiple systems, enabling sample-efficient estimation \citep{zhang2024sampleefficient} and certainty-equivalent control \citep{lee2024regret}. We focus on a policy representation that captures the modes to stabilize (i.e., the unstable modes). In particular, it carries a physical interpretation as it spans the system's left unstable subspace.

\subsection{Notation}
We use $\rho(\cdot)$ and $\sigma_{\min}(\cdot)$ to denote the spectral radius and the minimum singular value of a matrix, respectively. $\|\cdot\|$ is the $\ell_2$ norm, $\|\cdot\|_{\psi_2}$ denotes the sub-Gaussian norm \citep{vershynin2018high}, and $\|\cdot\|_F$ is the Frobenius norm of a matrix. $\trace{(\cdot)}$ is the trace function. $\mathbb{S}^{d-1}$ denotes the unit sphere. $\kappa(A)$ denotes the condition number of the matrix with the eigenvectors of $A$ as columns. We use $\mathrm{col}(A)$ to denote the subspace spanned by the columns of $A$. We use the big-O notation $\mathcal{O}(\cdot)$ to omit constants and $\widetilde{\mathcal{O}}(\cdot)$ to omit logarithmic factors in the argument. Unless otherwise stated, expectation is always taken with respect to the initial state.

\section{Problem  Formulation}\label{sec:problem_formulation}

Consider the discrete-time linear time-invariant (LTI) system 
\begin{align}\label{eq:LTI_sys mb}
    x_{t+1} = Ax_t + Bu_t, \text{ for }t=0,1,2,\hdots, 
\end{align}
where $x_{t} \in \mathbb{R}^{\dx}$ denotes the state and $u_t \in \mathbb{R}^{\du}$ the control input. We assume that the initial state $x_0$ is drawn according to a zero mean and isotropic distribution, i.e., $\E[x_0] = 0$, $\E[x_0x^\top_0] = I$, with $\|x_0\| \leq \mu_0$ and $\|x_0\|_{\psi_2} \leq \mu_\psi$. Let $\{\lambda_1,\lambda_2, \ldots, \lambda_{\dx}\}$, with $|\lambda_1| \geq  \ldots |\lambda_\ell| >  1 >  |\lambda_{\ell +1}| \geq  \ldots  \geq |\lambda_{\dx}|$, denote the eigenvalues of the drift matrix $A$. We focus on the setting where the system matrix $A$ is open-loop unstable, with $\ell \ll \dx$ unstable modes $\{\lambda_1, \ldots, \lambda_\ell\}$. We assume that \eqref{eq:LTI_sys mb} is stabilizable, which ensures the existence of a state feedback gain $K \in \mathbb{R}^{\du \times \dx}$ such that $\rho(A + BK) < 1$. 

\vspace{0.2cm}
\noindent \textbf{Goal:} Construct a stabilizing controller $K$ that defines a linear policy of the form $u_t = K x_t$, using policy gradient methods \citep{fazel2018global}, without requiring access to the system matrices $(A, B)$.

\subsection{Discounted Linear Quadratic Regulator Problem}
Given a ``discount factor'' $\gamma \in (0,1]$, the discounted LQR problem is described as follows:
\begin{align}\label{eq:discounted_LQR mb}
    &\text{min}_{K \in \mathcal{K}}\left\{J^\gamma (K) := \E\left[ \sum_{t=0}^{\infty} \gamma^t x^\top_t \left(Q  + K^\top RK\right)x_t \right]\right\}, \text{ subject to } \eqref{eq:LTI_sys mb} \text{ with } u_t = Kx_t,
\end{align}
where $\mathcal{K} := \{K\mid \rho(A+BK) < 1\}$ denotes the set of stabilizing controllers, and $(Q, R)$ are positive definite matrices. It is important to emphasize that in our problem setup, the cost matrices $(Q, R)$ are ``artificial" design parameters that will be used in the implementation of our solution method. Our goal \emph{is not} to learn an optimal control policy with respect to a specific cost, but rather to learn a controller that ensures the stability of \eqref{eq:LTI_sys mb}. By rescaling the state $x_t$ by $\gamma^{t/2}$, i.e., $\tilde{x}_t = \gamma^{t/2}x_t$, the discounted LQR problem \eqref{eq:discounted_LQR mb} is equivalent to  %\JA{It's still weird/bordering on wrong to use the same $J^\gamma$ for (2) and $(3)$. $J^{\gamma}$ is a function not an optimization problem. I think you need to say something along the lines of the ``(3) is the form of the problem we work with from this point onward''. }
\begin{align}\label{eq:LQR_with_damped_sys mb}
    &\text{min}_{K \in \mathcal{K}^\gamma} \left\{J^{\gamma}(K) :=  \E\left[ \sum_{t=0}^{\infty} \tilde{x}^\top_t\left(Q + K^\top RK\right)\tilde{x}_t \right]\right\}, \text{ subject to } \tilde{x}_{t+1} = (A^{\gamma}+ B^{\gamma}K)\tilde{x}_t,
\end{align}
where $\mathcal{K}^\gamma := \{K \mid \rho(A^\gamma+B^\gamma K) < 1\}$, with damped system matrices $A^{\gamma} := \sqrt{\gamma}A$, $B^\gamma := \sqrt{\gamma}B$. 

Note that by setting $\gamma$ sufficiently small, in particular, $\gamma < 1/\rho^2(A)$, the trivial controller $K \equiv 0$ stabilizes the underlying discounted LQR problem. However, such a control gain may not be stabilizing for the original system (i.e., for $\gamma = 1$). In fact, what allows us to design a stabilizing controller by solving a sequence of discounted LQR problems is the appropriate incremental update of $\gamma$. Let $\gamma_j$ denote the discount factor at iteration $j \in \mathbb{N}$.  \cite{zhao2024convergence}  showed that by repeating the following process (while $\gamma_{j+1} <1$): 
\begin{enumerate}
    \item Compute a controller $K_{j+1}$ by solving~\eqref{eq:discounted_LQR mb} such that $J^{\gamma_j}(K_{j+1}) \leq \bar{J}$,
    \item Update the discount factor: $\gamma_{j+1} = (1+\xi\alpha_j)\gamma_j$,
\end{enumerate}
 a stabilizing controller $K \in \mathcal{K}$ is found within a finite number of iterations of the above process. Here, $\xi \in (0,1)$ is the decay factor, $\bar{J}$ is a uniform bound of the discount LQR cost, and $\alpha_j > 0$ is the discount factor update rate. We elaborate on the role and selection of each of these quantities in Section~\ref{sec:LTS on the unstable subspace}, where we introduce our method for learning a stabilizing controller on the unstable subspace. For now, it is important to highlight that such explicit discount method comes with a sample complexity that scales quadratically with the system's state dimension, i.e., $\widetilde{\mathcal{O}}(\dx^2 \du)$, thus limiting its applicability for high-dimensional systems where data collection is difficult and thus data is scarce (e.g., robot manipulation \citep{billard2019trends}).

However, high-dimensional unstable systems often possess only a small number of unstable modes, as in our setting of interest $\ell \ll \dx$. That observation motivates the following question: \emph{Can we apply the discount method directly on the unstable subspace, aiming to stabilize only the small portion of the state space associated with the unstable dynamics?} We answer this question in the affirmative. For this purpose, we introduce a linear parameterization of $K$ for stabilizing the unstable modes of \eqref{eq:LTI_sys mb} independently from its stable dynamics.

\subsection{Stabilizing Only the Unstable Modes}

Let $\Omega := [\Phi \;\ \Phi_\perp]$ be an orthonormal basis of $\mathbb{R}^{\dx}$, where the columns of $\Phi \in \mathbb{R}^{\dx \times \ell}$ span the \emph{left} eigenspace corresponding to the unstable modes of $A$. We  refer to this as the ``left unstable subspace of $A$'', and  to $\Phi$ as the ``unstable subspace representation''. Hence, we can write the following:
$$
\Omega^\top A\Omega = \begin{bmatrix}
        A_u  & \\
        \Delta & A_s
\end{bmatrix}, \text{ with } A_u = \Phi^\top A \Phi, ~\Delta = \Phi^\top_\perp A \Phi, \text{ and } A_s = \Phi^\top_\perp A \Phi_\perp,
$$
where $A_u$ represents the unstable dynamics of $A$, as it has the spectrum of the Jordan blocks corresponding to the unstable eigenvalues of $A$. On the other hand, $A_s$ inherits all stable modes of $A$. The matrix $\Delta$ represents the ``coupling" of the stable and unstable dynamics arising from the $\mathrm{col}(\Phi) 
$ $\oplus$ $\mathrm{col}(\Phi_\perp)$ decomposition. We also note that $\Delta \equiv 0$ when $A$ is symmetric.   
\vspace{0.15cm}

\noindent \textbf{Controller representation:} Suppose that $K$ is linearly decomposed into a low-dimensional control gain $\theta \in \mathbb{R}^{\du \times \ell}$ and the left unstable subspace representation $\Phi$, namely, $K = \theta \Phi^\top$. The closed-loop system matrix $A+BK$ can then be written as  
$$
    A + BK = \Omega\begin{bmatrix}
        A_u + B_u\theta  & \\
        \Delta + B_s\theta & A_s
\end{bmatrix}\Omega^\top := ~\Omega \bar{A} \Omega^\top, \text{ where } B_u = \Phi^\top B \text{ and } B_s = \Phi^\top_\perp B.
$$

From the above decomposition, it suffices to stabilize the low-dimensional unstable dynamics described by $(A_u,B_u)$ to guarantee the stability of $(A,B)$.  Hence, one may reduce the problem of stabilizing $(A,B)$ through designing $K$, to that of stabilizing $(A_u,B_u)$ by finding a low-dimensional controller $\theta$ such that  $\rho(A_u+B_u\theta) < 1$. Intuitively, the reduction in the control space should also yield a reduction in the sample complexity of learning the stabilizing controller.     

\begin{remark}
One might naturally ask: “Why not decompose $K$ with respect to the right unstable subspace of $A$ instead?” We emphasize that doing so introduces the coupling term $\Delta$ in the top-right block of the decomposition of $A$, as it appears in \cite{hu2022sample, zhang2024learning}. This disrupts the triangular structure of $\bar{A}$ and thus $\Delta$ incurs a bias in the spectral radius of the closed-loop system matrix. As a result, the condition of stabilizing $(A,B)$ via the stabilization of $(A_u,B_u)$ would only be guaranteed if $\|\Delta\|$ is sufficiently small. Therefore, if $\|\Delta\|$ is large, its inevitable effect in $\rho(\bar{A})$ due to the right unstable subspace parameterization would lead to an inflation in the sample complexity or it may even prevent us from stabilizing the \eqref{eq:LTI_sys mb}, as seen in \cite{hu2022sample, zhang2024learning}.  That is not the case in this work since we operate with the left unstable subspace.
\end{remark}

\subsection{Low-Dimensional Discounted LQR Problem}

Given the left unstable subspace representation $\Phi$, let $z_t \in \mathbb{R}^{\ell}$ denote the low-dimensional state that represents $x_t$ on the subspace spanned by the columns of $\Phi$, i.e., $x_t = \Phi z_t$. The low-dimensional unstable dynamics of \eqref{eq:LTI_sys mb} evolve according to the system
\begin{align}\label{eq:low-dimensional-LTI mb}
    z_{t+1} = A_uz_t + B_uu_t,\quad  \text{ for } t = 0,1,2,\ldots, 
\end{align}  
where $z_0$ is also drawn from a zero mean and isotropic distribution since $\Phi$ is orthonormal. We can now write the discounted LQR problem on the unstable subspace in the form of~\eqref{eq:LQR_with_damped_sys mb} as follows:
\begin{align}\label{eq:cost low dimensional mb}
&\text{min}_{\theta \in \Theta^\gamma} \left\{J^{\gamma}(\theta, \Phi) :=  \E\left[ \sum_{t=0}^{\infty} z^\top_t\left( \Phi^\top Q \Phi + \theta^\top R\theta \right) z_t \right]\right\}, \text{ subject to } z_{t+1} = (A^\gamma_u + B^\gamma_u\theta)z_t,
\end{align}
where $\Theta^\gamma := \{\theta \mid \sqrt{\gamma}\rho(A^\gamma_u + B^\gamma_u \theta) < 1\}$ is the set of stabilizing controllers for the damped unstable dynamics $A^\gamma_u := \sqrt{\gamma}A_u$ and $B^\gamma_u := \sqrt{\gamma}B_u$. Let $\nabla J^\gamma (\theta, \Phi)$ be the gradient with respect to $\theta$, then
$$
    \nabla J^\gamma (\theta, \Phi) = \nabla J^\gamma(\theta \Phi^\top) \Phi = 2E_\theta\Sigma_\theta,
$$
with 
\begin{equation*}
E_\theta := (R+B^{\gamma \top}_u P^\gamma_\theta B^{\gamma}_u)\theta + B^{\gamma\top}_u P^\gamma_\theta A^{\gamma}_u,  \text{ where } P^\gamma_\theta = \Phi^\top Q \Phi + \theta^\top R \theta + (A^\gamma_u + B^\gamma_u \theta )^\top  P^\gamma_\theta (A^\gamma_u + B^\gamma_u \theta),
\end{equation*}
and closed-loop state covariance $\Sigma_\theta := \E\left[\sum_{t=0}^\infty z_tz^\top_t\right]$. With a slight abuse of notation, we write $J^\gamma(\theta) := J^\gamma(\theta,\Phi)$ and note that the discounted LQR cost can be written as $J^\gamma(\theta) = \trace{(P^\gamma_\theta)}$. 

\begin{definition}\label{definition:sublevel-set of stabilizing controllers} Given a discount factor $\gamma \in (0,1]$ and scalar $\mu_s > 0$. Let $\mathcal{S}^\gamma_\theta$ denote a sublevel set of $\Theta^\gamma$, $\mathcal{S}^\gamma_\theta \subseteq \Theta^\gamma$, with $\mathcal{S}^\gamma_\theta := \left\{ \theta \mid J^\gamma(\theta) - J^\gamma(\theta^\star) \leq \mu_s\left(J^\gamma(\theta_0) - J^\gamma(\theta^\star)\right)\right\}$, where $\theta^\star$ is the optimal controller of the underlying low-dimensional discounted LQR problem \eqref{eq:cost low dimensional mb}.
\end{definition}

Similarly, $\mathcal{S}^\gamma_K$ denotes the sublevel set of $\mathcal{K}^\gamma$ for the high-dimensional LQR problem \eqref{eq:discounted_LQR mb}. We use $J^\gamma_\star$ to denote the optimal cost. Let $\phi$, $\nu_\theta$, $L_\theta$, $L_K$ and $\mu_{\text{PL}}$ be positive constants. The following properties of $J^\gamma(\theta)$ and $J^\gamma(K)$ hold in their respective stabilizing sublevel sets, $\mathcal{S}^\gamma_\theta$ and $\mathcal{S}^\gamma_K$. 

\begin{lemma}\label{lemma:uniform bounds and lipschitz mb} Given high-dimensional and low-dimensional stabilizing controllers $K, K^\prime \in \mathcal{S}^\gamma_K$ and $\theta, \theta^\prime \in \mathcal{S}^\gamma_\theta$, respectively.  It holds that $\|\nabla J^\gamma(K)\| \leq \phi, \|\theta\| \leq \nu_\theta$, and 
\begin{align*}
&\left\|\nabla J^\gamma\left(\theta\right)-\nabla J^\gamma(\theta^\prime)\right\|_F \leq L_\theta\|\theta - \theta^\prime\|_F, \quad \left\|\nabla J^\gamma\left(K\right)-\nabla J^\gamma(K^\prime)\right\|_F \leq L_K\|K - K^\prime\|_F.
\end{align*}
\end{lemma}
\begin{lemma}\label{lemma:gradient dominance mb}
Given a stabilizing controller $\theta \in \mathcal{S}^\gamma_\theta$. It holds that $\|\nabla J^\gamma(\theta)\|_F^2 \geq \mu_{\text{PL}} (J^\gamma(\theta)-J^\gamma(\theta^\gamma_\star))$. 
\end{lemma}
 
\begin{remark} Lemmas \ref{lemma:uniform bounds and lipschitz mb} and \ref{lemma:gradient dominance mb} were originally proved by \cite{fazel2018global} and subsequently revisited by \cite{gravell2020learning}, where the explicit expression of the problem dependent constants $\phi$, and $\nu_\theta$ are provided. We define here $\phi$, $\nu_\theta$, $L_\theta$, $L_K$, and $\mu_{\text{PL}}$ as the uniform bound over the set of all stabilizing controllers, i.e., either $\mathcal{S}^\gamma_\theta$ or $\mathcal{S}^\gamma_\theta$, for any $\gamma \in (0,1)$.  
\end{remark}

We conclude this section by recalling that our setting is model-free, and therefore the left unstable subspace representation $\Phi$ cannot be accessed directly. In the following section, we show that an accurate estimate of $\Phi$, denoted by $\widehat{\Phi}$, can be recovered when a sufficient amount of trajectory data is collected. The accuracy of this estimate is quantified using the subspace distance between the column spaces of $\widehat{\Phi}$ and $\Phi$, as defined in \cite{stewart1990matrix}. 

\begin{definition}\label{def:subspace_distance} 
Let $\widehat{\Pi} = \widehat{\Phi}\widehat{\Phi}^\top$ and $\Pi = \Phi\Phi^\top$ be orthogonal projectors onto the column spaces of $\widehat{\Phi}$ and $\Phi$, respectively. The subspace distance between $\Phi$ and $\hat{\Phi}$ is $d(\widehat{\Phi}, \Phi) \triangleq\|\hat{\Phi}^{\top} \Phi_{\perp}\| = \|\widehat{\Pi} - \Pi\|$. 
\end{definition}

\section{Learning the Left Unstable Representation} \label{sec: learning the left unstable reprensentation}

\noindent \textbf{Sampling from the adjoint system:} To learn an estimate of the left unstable subspace representation, we proceed by first collecting data from the autonomous adjoint system of \eqref{eq:LTI_sys mb}, i.e., $x_{t+1}=A^\top x_t$ \citep{kouba2020adjoint}. To do so, we perform element-wise computations with the adjoint operator while forward simulating \eqref{eq:LTI_sys mb} accordingly. Note that for any real-valued matrix $A \in \mathbb{R}^{\dx \times \dx}$ and vectors $x,y \in \mathbb{R}^{\dx}$, we have $ \langle Ax,y\rangle = \langle x, A^\top y\rangle$. Therefore, by playing \eqref{eq:LTI_sys mb} with zero input $u_0 \equiv 0$ and initial condition $x_0 = e_i$, where $\{e_i\}_{i=1}^{\dx}$ is the canonical basis of $\mathbb{R}^{\dx}$, we collect and store $e^+_i = Ae_i$ to obtain
$$ (A^\top x)_i := \langle e_i, A^\top x\rangle = \langle e^+_i,x\rangle, \forall i \in \{1,2,\ldots,\dx\},$$
which implies that the next adjoint state is $x_{t+1} = \left[ x^\top_t e^+_1  \ldots  x^\top_t e^+_{\dx} \right]^\top$. Hence, the next adjoint state is derived from the previous state $x_t$ and samples $\{e^+_i\}_{i=1}^{\dx}$ collected by interacting with \eqref{eq:LTI_sys mb}. 

\vspace{0.2cm}

\noindent \textbf{Goal:} Construct an estimation for the left unstable subspace of $A$ from $T$ data samples collected from the autonomous adjoint system $D = \left[x_1, x_2, \ldots, x_T   \right] \in \mathbb{R}^{\dx\times T}$.

\vspace{0.2cm}

\noindent \textbf{Estimating the left unstable subspace:} We proceed by computing the singular value decomposition $D = U\Sigma V^\top$. An estimation of the orthonormal basis of the right unstable subspace of $A^\top$ (or left unstable subspace of $A$) is obtained from the range of the top $\ell$ columns of $U$, i.e., $\widehat{\Phi} = [u_1, \ldots, u_\ell]$. We now show that $d(\widehat{\Phi}, \Phi)$ becomes sufficiently small as the trajectory length $T$ increases. To establish this result, we leverage a similar approach to \cite[Theorem 5.1]{zhang2024learning}, with two key distinctions: our setting accommodates \emph{non-diagonalizable} system matrices $A$, and our estimation focuses on the \emph{left unstable subspace representation}.

Let $\Psi \in \mathbb{R}^{\dx \times (\dx - \ell)}$ denote an orthonormal basis for the left stable subspace of $A$, and define $\Xi = \left[\Phi \;\ \Psi \right]$, which contains the left eigenvectors corresponding to the unstable and stable modes of $A$. These may include generalized eigenvectors, accounting for $A$ to be non-diagonalizable. Hence, there exists matrices $\Lambda_u \in \mathbb{R}^{\ell \times \ell}$ and $\Lambda_s \in \mathbb{R}^{(\dx - \ell)\times (\dx - \ell)}$ with the same spectrum of the Jordan blocks corresponding to the unstable and stable modes of $A$, respectively. As a result, we can write
$$
    A^\top \left[\Phi\;\ \Psi\right] =  \left[\Phi\;\ \Psi\right] \begin{bmatrix} \Lambda_u &  \\
                                                & \Lambda_s   
                        \end{bmatrix}, \text{ and define } \Xi ^{-1}:= S = [S_1^\top\;\ S^\top_2 ]^\top  \text{ to obtain }
$$
$$
    D = \Xi SD = \left[\Phi \;\   \Psi \right]\begin{bmatrix}
        S_1D\\ S_2D
    \end{bmatrix}  = \Phi D_1 + \Psi D_2 = D_u + D_s,
$$
where $D_1 = S_1D$ and $D_2 = S_2D$. We note that $D=D_u+D_s$ is composed of $D_u = \Phi D_1$ that comes from the unstable dynamics of $A$ and $D_s = \Psi D_2$ that depends on the stable counterpart. 

Let us first analyze $D_u$ by using the singular value decomposition of $D_1$, i.e., $D_u = \Phi D_1 = \Phi U_1\Sigma_1 V^\top_1$, with $U_1 \in \mathbb{R}^{\ell \times \ell}$, $\Sigma_1 \in \mathbb{R}^{\ell \times \ell}$, and $V_1 \in \mathbb{R}^{T \times \dx}$. Note that $\widehat{\Pi}$ is the projector onto the subspace spanned by the top $\ell$ columns of $U$, whereas $\Pi$ projects onto the subspace spanned by the columns of $\Phi U_1$. The following lemma  characterizes the distance between these subspaces.

\begin{lemma}\label{lemma:subspace-distance} Let $\sigma_\ell$ be the $\ell$-th singular value of $D_u$ and $\hat{\sigma}_{\ell+1}$ the $\ell+1$-th singular value of $D$. Then,
$$
     d(\widehat{\Phi},\Phi) \leq  \frac{\sqrt{2\ell}\sqrt{T}(\dx - \ell)\mu_0}{(\sigma_\ell - \hat{\sigma}_{\ell+1})(1 - |\lambda_{\ell+1}|)},
$$ 
where $d(\cdot)$ is the subspace distance as defined in Definition \ref{def:subspace_distance}.
\end{lemma}

\vspace{0.2cm}

The proof follows directly from Davis-Kahan theorem \citep{davis1970rotation} along with the following upper bound on $\|D_2\|$: 
$$
\hat{\sigma}_{\ell+1} \leq \|D_2\| \leq \sqrt{T}\sum_{i = \ell+1}^{\dx} \sum_{t=1}^{T}|\lambda_i|^t\|x_0\| \leq   \frac{\sqrt{T}(\dx - \ell)\mu_0}{1 - |\lambda_{\ell+1}|}.
$$
We refer the reader to Appendix \ref{Appendix:learning the left unstable subspace} for more details. It remains to characterize the scaling of $\sigma_\ell$ with respect to the trajectory length $T$.

\begin{lemma}\label{lemma: lower bound on sigma_l} Suppose that the number of samples collected from the adjoint system scales according to $T = \mathcal{O}\left(\log(\ell^7/\delta^3_\sigma)/\log(|\lambda_\ell|
)\right)$ for some $\delta_\sigma \in (0,1)$. Then, it holds that 
\begin{align*}
    \sigma_\ell \geq \frac{\sqrt{C_{\sigma}}|\lambda_\ell|^{T}\delta_{\sigma}}{2\sqrt{2}C_{\psi}\ell^{5/2}T^{3/2}}, \text{ with probability $1-4\delta_\sigma$}, \hspace{-0.05cm}\text{ where $C_{\sigma} \hspace{-0.05cm}= \hspace{-0.05cm}\mathcal{O}(1) \text{ and } C_{\psi} \hspace{-0.05cm}= \hspace{-0.05cm}\mathcal{O}(1)$.}
\end{align*}
 
\end{lemma}

We detail the proof in Appendix \ref{Appendix:learning the left unstable subspace}. For now, it is important to note that if $|\lambda_\ell| \gg 1$, then as $T \to \infty$, the subspace distance $d(\widehat{\Phi},\Phi) = \frac{\mathcal{O}({T}^2)}{\mathcal{O}(|\lambda_\ell|^{T}) - \mathcal{O}({T}^2)}$ goes to zero, with high probability.  Below, we formalize the non-asymptotic guarantees of learning the left unstable subspace representation.

\begin{theorem}\label{theorem:learning the left unstable subspace mb} Suppose that the amount of trajectory data for learning the left unstable subspace representation scales according to $T = \mathcal{O}\left(\log\left( \frac{\ell^{7}(\dx - \ell)\mu_0}{(1 - |\lambda_{\ell+1}|)\varepsilon \delta^3_\sigma} \right)/\log(|\lambda_{\ell}|)\right)$, for some small accuracy $\varepsilon>0$ and $\delta_\sigma \in (0,1)$. Then, it holds that $d(\widehat{\Phi},\Phi)\leq \varepsilon$, with probability $1 - 4\delta_\sigma$.  
\end{theorem}

Let us now take a moment to explain this result. First, observe that the required number of samples $T$ depends only \emph{logarithmically} on the problem ambient dimension $\dx$ and the number of unstable modes $\ell$. The main bottleneck in learning the left unstable subspace arises when the least unstable mode is close to marginal stability, i.e., $|\lambda_\ell| \approx 1$. Conversely, Theorem~\ref{theorem:learning the left unstable subspace mb} states that the estimation becomes easier as the system becomes more explosive (i.e., $|\lambda_\ell| \gg 1$).

In addition, while the constant $C_\sigma$ does not scale with $\ell$ or $T$, it is sensitive to the spectral properties of the system. In particular, it depends on the spectral norm of the Jordan matrix $\Lambda = \texttt{blkdiag}(\Lambda_1,\ldots, \Lambda_n)$, with $n$ being the number of distinct eigenvalues of $A$. Each Jordan block takes the form $\Lambda_i = \texttt{diag}(\lambda_i,\ldots,\lambda_i) + {N}_i$, where ${N}_i$ is a nilpotent matrix with ones on the first superdiagonal, if the geometric multiplicity of $\lambda_i$, denoted by $\text{gm}(\lambda_i)$, is equal to one. As discussed in \cite{sarkar2019near}, the estimation of the unstable component becomes inconsistent when the geometric multiplicity of the unstable eigenvalues is greater than one. In our setting, this effect deflates $C_\sigma$ which in turn increases the number of samples $T$ when $A$ contains unstable modes with geometric multiplicity greater than one.

\begin{wrapfigure}{r}{0.65\textwidth}
    \centering
    \includegraphics[width=0.6\textwidth]{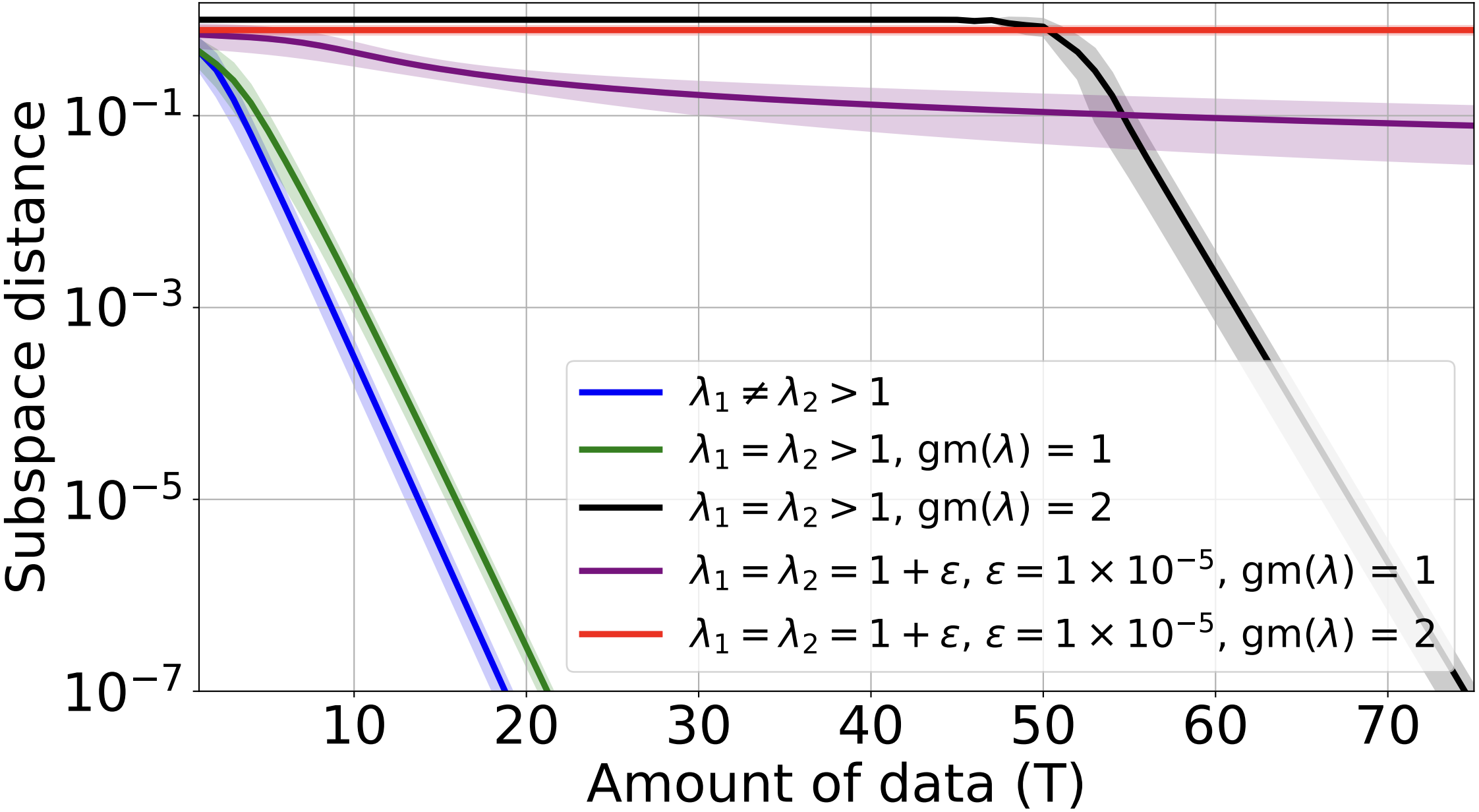}
    \caption{Subspace distance $d(\widehat{\Phi},\Phi)$ with respect to data $(T)$.} 
    \label{fig:subspace_dist}
\end{wrapfigure}

Figure \ref{fig:subspace_dist}, illustrates these trends for a simple example with $\dx = 3$ states and $\ell = 2$ unstable modes. The plot depicts the mean and the standard deviation for $10$ different random initial conditions. Notably, learning the unstable subspace for a diagonalizable matrix (blue curve) requires roughly the same amount of data as for a non-diagonalizable case with $\text{gm}(\lambda)=1$ (green curve). In contrast, as the least unstable mode gets close to one, successful estimation becomes infeasible.

\begin{remark} We emphasize that the guarantees for learning the right unstable subspace of $A$ presented in \cite{zhang2024learning} could not be directly applied in our setting. This is because their order of $T$ depends inversely on the gap between the unstable modes, which becomes problematic when the system matrix is non-diagonalizable, as the gap goes to zero and the amount of data grows prohibitively large. Moreover, such a dependence on the spectral gap appears to be counterintuitive and does not align with the results illustrated in Figure~\ref{fig:subspace_dist}. 
\end{remark}

\section{Learning to Stabilize on the Unstable Subspace}\label{sec:LTS on the unstable subspace}

We now introduce our approach for learning to stabilize (LTS) by operating on the system's unstable subspace. This method combines the unstable subspace representation learning, discussed in the previous section, with the discounted LQR method applied directly on the learned subspace. Specifically, our goal is to learn a low-dimensional controller $\theta \in \mathcal{S}^1_\theta$ that stabilizes the unstable dynamics $(A_u, B_u)$. This is accomplished by solving a series of discounted LQR problems with PG.

Recall that PG requires access to the gradient $\nabla J^\gamma(\theta, \Phi)$. However, because we operate in a model-free setting, this gradient cannot be computed directly. To address this, we use a zeroth-order gradient estimation method \citep{flaxman2004online, spall2005introduction}, which yields a gradient estimate denoted by $\widehat{\nabla}J^\gamma(\theta,\widehat{\Phi})$. This estimation is performed by collecting trajectory data from the original system \eqref{eq:LTI_sys mb}, projected onto the estimated left unstable subspace via $\widehat{\Phi}$, i.e., using $z_t = \widehat{\Phi}^\top x_t$.

Before introducing the zeroth-order gradient estimation and its finite-sample guarantees, we first provide an upper bound on the error between $\nabla J^\gamma(\theta, \Phi)$ and $\nabla J^\gamma(\theta, \widehat{\Phi})$.

\begin{lemma}\label{lemma: error gradient misspecification}
Suppose that $\theta \in \mathcal{S}^\gamma_\theta$. Then, 
$$
    \left\|\nabla J^\gamma (\theta, \Phi) - \nabla J^\gamma (\theta, \widehat{\Phi})\right\|_F \leq C_{\Phi} d(\widehat{\Phi}, {\Phi}), \text{ with } C_\Phi = \sqrt{\ell}\left( (L_K\nu_\theta+\phi) \sqrt{2\ell} + \phi\right).
$$
\end{lemma}

Note that the error in the gradient incurred by the learned representation, i.e., $\widehat{\Phi}$, can be made arbitrarily small, provided that $d(\widehat{\Phi},\Phi)$ is sufficiently small. The proof of this lemma follows from Lemmas \ref{lemma:uniform bounds and lipschitz mb} and \ref{lemma:gradient dominance mb}, combined  with the upper bound $\|\widehat{\Phi} - \Phi\| \leq \sqrt{2\ell}d(\widehat{\Phi}, \Phi)$ from \cite[Corollary 5.3]{hu2022sample}. Additional details can be found in Appendix \ref{appendix: linear decomposition of the control policy}.

\vspace{-0.1cm}
\subsection{Gradient and Cost Estimation}

The zeroth-order gradient estimation method is standard and has been widely adopted for policy gradient estimation in model-free LQR \citep{fazel2018global, malik2019derivative}. Next, we define the two-point zeroth-order estimation.
\begin{align*}
    \widehat{\nabla} J^\gamma (\theta, \widehat{\Phi}) := \frac{1}{2rn_s}\sum_{i=1}^{n_s} \left(V^{\gamma,\tau}(\theta_{1,i},z^i_0) - V^{\gamma,\tau}(\theta_{2,i},z^i_0)
  \right)U_i,
\end{align*}
where $U_i$ is randomly drawn from a uniform distribution on the sphere $\sqrt{\ell \du}\mathbb{S}^{\ell\du-1}$. In addition, we have that $\theta_{1,i} = \theta + rU_i$, $\theta_{2,i} = \theta - r U_i$, with $r > 0$ denoting the smoothing radius and $n_s$ the number of rollouts (or trajectories). Let $\tau > 0$ denote the time horizon. The finite-horizon value function $V^{\gamma,\tau}(\theta, z_0)$ is defined as follows:
\begin{align*}
    V^{\gamma,\tau}(\theta,z_0) := \sum_{t=0}^{\tau-1} \gamma^tz^\top_t\left(\widehat{\Phi}^\top Q\widehat{\Phi} + \theta^\top R \theta \right)z_t, 
\end{align*}
with $\{z_t\}^{\tau-1}_{t=0} = \{\widehat{\Phi}^\top x_t\}_{t=0}^{\tau-1}$ and $\{x_t\}^{\tau-1}_{t=0}$ being the trajectory data from \eqref{eq:LTI_sys mb} with $u_t = \theta \widehat{\Phi}^\top x_t$. 

\begin{lemma}\label{lemma:zeroth-order estimation mb}
Let $\zeta$ and $\varepsilon_\tau$ be positive scalars. Suppose that $n_s  = \mathcal{O}(\zeta^4\mu^4_\psi\log^6(\ell))\ell$, $\tau = \mathcal{O}( \log(1/\varepsilon_\tau))$ and $r = \mathcal{O}(\sqrt{\varepsilon_\tau})$. It holds with probability at least $1 - c_1(\ell^{-\zeta}+n^{-\zeta}_s-n_se^{-\ell/8} - e^{-c_2n_s})$ that 
$$
     \|\widehat{\nabla} J^\gamma(\theta,\widehat{\Phi})\|^2_F \leq C_{\text{est,1}} \|\nabla J^\gamma(\theta) \|^2_F  + C_{\text{est},1}C^2_\Phi d(\widehat{\Phi},{\Phi})^2  +\varepsilon^2_\tau,
$$
$$
    \langle \nabla J^\gamma(\theta), \widehat{\nabla}J^\gamma(\theta,\widehat{\Phi})  \rangle   \geq  C_{\text{est},2} \left\|\nabla J^\gamma (\theta)\right\|^2_F-C_{\text{est},3} C^2_\Phi d(\widehat{\Phi}, {\Phi})^2 - C_{\text{est},4}\varepsilon^2_\tau,
$$
with positive scalars $c_1$ and $c_2$. $C_{\text{est},1}$, $C_{\text{est},3}$, and $C_{\text{est},4}$ scale as $\mathcal{O}(\du \ell \log^2(\ell))$, and $C_{\text{est},2} = \mathcal{O}(1)$. 
\end{lemma}

The proof for this lemma follows from \cite[Section V]{mohammadi2020linear} and Lemma \ref{lemma: error gradient misspecification}. Lemma~\ref{lemma:zeroth-order estimation mb} states that if $\Phi$ is accurately estimated, the smoothing parameter $r$ is chosen sufficiently small, and both the time horizon $\tau$ and the number of rollouts $n_s$ are sufficiently large, then $\|\widehat{\nabla} J^\gamma(\theta,\widehat{\Phi})\|^2_F = \mathcal{O}\left(\|{\nabla} J^\gamma(\theta)\|^2_F\right)$ and $\langle \nabla J^\gamma(\theta,{\Phi}), \widehat{\nabla}J^\gamma(\theta,\widehat{\Phi})  \rangle = \mathcal{O}\left(\left\|\nabla J^\gamma (\theta,{\Phi})\right\|^2_F\right)$, with high probability. This result is crucial to establish the linear convergence of PG for \eqref{eq:discounted_LQR mb}.

Moreover, let $\widehat{J}^{\gamma,\tau}(\theta,\widehat{\Phi}) = \frac{1}{n_c}\sum_{i=1}^{n_c} V^{\gamma,\tau}(\theta,z^i_0)$ be the estimated cost with $n_c$ rollouts. We provide the following lemma to control $|J^\gamma(\theta)-\widehat{J}^{\gamma,\tau}(\theta,\widehat{\Phi})|$; the proof is deferred to Appendix \ref{appendix:cost and gradient estimation}.

\begin{lemma}\label{lemma:error cost}

Given a stabilizing controller $\theta \in \mathcal{S}^\gamma_\theta$ and $\delta_\tau \in (0,1)$. Suppose that  
\begin{align*}
    \tau \geq \tau_0 := \frac{J^\gamma(\theta,\widehat{\Phi})}{\sigma_{\min}( Q)}\log\left(\frac{8(J^{\gamma}(\theta,\widehat{\Phi}))^2\mu^2_0}{\sigma_{\min}(Q)J^\gamma(\theta)}\right), n_c \geq 8\mu^2_0\log\left(2/\delta_\tau\right), \text{ and } d(\widehat{\Phi},\Phi) \leq J^\gamma(\theta)/(4\ell\sqrt{\ell}C_{\text{cost}}).
\end{align*}
Then, it holds that  $|\widehat{J}^{\gamma,\tau}(\theta,\widehat{\Phi}) - J^\gamma(\theta)|\leq \frac{1}{2}J^\gamma(\theta)$, with probability $1-\delta_\tau$. $C_{\text{cost}}$ is polynomial in the problem parameters $\|A\|$, $\|B\|$, $\|Q\|$, $\|R\|$, and $\nu_\theta$.   
\end{lemma}

\subsection{Discounted LQR on the Unstable Subspace}

With the gradient and cost estimation results in place, we are now ready to present our discounted LQR method on the unstable subspace for learning to stabilize the system’s unstable dynamics. As a starting point, we assume access to an upper bound on the largest eigenvalue, namely, $|\lambda_1| \leq \bar{\lambda}_1$. This assumption is necessary to initialize the discount factor as $\gamma_0 < 1/\bar{\lambda}_1^2$, which ensures that the initial controller $\theta_0 \equiv 0$ stabilizes the corresponding damped system.

To ensure that the discount factor reaches one within a finite number of iterations, we adopt the explicit discount scheme proposed in \cite{zhao2024convergence}. In particular, $\gamma_{j+1} = (1+\xi \alpha_j)\gamma_j$, where $\xi \in (0,1)$ is the decay factor and the update rate $\alpha_j$ is given by
\begin{align} \label{eq:alpha_update}
    \alpha_j = \frac{3\sigma_{\min}\left( \widehat{\Phi}^\top Q \widehat{\Phi} + \theta^\top_jR\theta_j \right)}{\frac{4}{3}\widehat{J}^{\gamma_j,\tau}(\theta_j,\widehat{\Phi}) - 3\sigma_{\min}\left( \widehat{\Phi}^\top Q \widehat{\Phi} + \theta^\top_jR\theta_j \right)}.
\end{align}

The update rule for the discount factor follows from the Lyapunov stability analysis of the low-dimensional damped system. Let $V(z_t) = z^\top_t P^\gamma_\theta z_t$ be a quadratic Lyapunov function for the damped dynamics $z_{t+1} = \sqrt{\gamma_{j+1}}(A_u + B_u\theta)z_t$, and define $\Delta V = V(z_{t+1}) - V(z_t)$. Hence, we have
\begin{align*}
     \Delta V= z^\top_t\left(\frac{\gamma_{j+1}}{\gamma_j}\left(P^\gamma_\theta - {\Phi}^\top Q {\Phi} - \theta^\top R \theta\right) - P^\gamma_\theta   \right)z_t,
\end{align*}
and thus $\frac{\gamma_{j+1}}{\gamma_j}\left(P^\gamma_\theta - \Phi^\top Q \Phi - \theta^\top R \theta\right) - P^\gamma_\theta \prec 0$ yields $\sqrt{\gamma_{j+1}}\rho(A_u + B_u\theta) < 1$. Sufficiently, we write
\begin{align*}
1 - \frac{\gamma_j}{\gamma_{j+1}} &\leq \sigma_{\min}({\Phi}^\top Q {\Phi} + \theta^\top R \theta)/\trace{(P^\gamma_\theta)} \leq \frac{3}{2}\sigma_{\min}(\widehat{\Phi}^\top Q \widehat{\Phi} + \theta^\top R \theta)/J^\gamma(\theta), 
\end{align*}
where the last inequality follows from Bauer-Fike theorem \citep{bauer1960norms} and making the subspace distance to satisfy  $d(\widehat{\Phi},\Phi)\leq \frac{\sigma_{\min}(\widehat{\Phi}^\top Q \widehat{\Phi}+\theta^\top R \theta)}{4\|Q\|\sqrt{2\ell}\kappa(\widehat{\Phi}^\top Q \widehat{\Phi}+\theta^\top R \theta)}$. Therefore, by invoking Lemma \ref{lemma:error cost}, we recover \eqref{eq:alpha_update}. As also discussed in \cite{zhao2024convergence}, the decay factor $\xi \in (0,1)$ ensures that the updated controller $\theta_{j+1}$ ``strongly" stabilizes the damped system $(A^{\gamma_{j+1}}_u, B^{\gamma_{j+1}}_u)$. In the following section, we show the role of $\xi$ in providing a uniform bound for $\sqrt{\gamma_{j+1}}\rho(A^{\gamma_{j+1}}_u + B^{\gamma_{j+1}}_u \theta_{j+1})$.   

We conclude this section by presenting the algorithm for learning a stabilizing controller for system \eqref{eq:LTI_sys mb} by operating on its unstable subspace. As previously discussed, we initialize the discount factor with $\gamma_0 < 1/\bar{\lambda}_1^2$ and choose $\xi \in (0,1)$. The data $D$ collected from the adjoint system is used in line 2 of Algorithm~\ref{alg:disc_LQR} to estimate the left unstable subspace representation, which in turn defines the discounted LQR problem over the learned subspace. The algorithm proceeds by solving a sequence of low-dimensional discounted LQR problems, while $\gamma_j < 1$ (lines 4–8). In particular, given a stabilizing controller $\theta_j$ for the damped system with factor $\gamma_j$, $N$ policy gradient iterations using the estimated gradient $\widehat{\nabla} J^{\gamma_j}(\theta, \widehat{\Phi})$ are performed (lines 5 and 6). The number of iterations $N$ is set to ensure that $J^{\gamma_j}(\theta) \leq \bar{J}$ (this is made explicit in the next section). The discount factor is updated with \eqref{eq:alpha_update} (line 8). Finally, Algorithm \ref{alg:disc_LQR} returns $K = \theta_{j+1} \widehat{\Phi}^\top \in \mathcal{S}^{1}_K$.

\begin{algorithm}
\caption{Learning to Stabilize on the Unstable Subspace } 
\label{alg:disc_LQR}
\begin{algorithmic}[1]
\State \textbf{Input:} $\gamma_0$, $\xi$, $N$, $\eta$, $D$\
\State \textbf{Compute} $D = U\Sigma V^\top$ and let $\widehat{\Phi}= \left[u_1,\ldots, u_\ell\right]$ be the top $\ell$ columns of $U$
\State \textbf{Initialize} $\theta_0 = 0$ and $j = 0$
\State \textbf{While} $\gamma_{j} < 1$ \textbf{do}\
\State \quad \textbf{Initialize} $\bar{\theta}_0 = {\theta}_j$ and \textbf{for} $n = 0,1,\ldots,N-1$ \textbf{do}
\State \quad \quad $\bar{\theta}_{n+1} = \bar{\theta}_n - \eta \widehat{\nabla}J^{\gamma_j}(\bar{\theta}_n, \widehat{\Phi})$
\State \quad \textbf{Let} ${\theta}_{j+1}= \bar{\theta}_N$ and   \textbf{compute} $\alpha_j$ as in Eq. \eqref{eq:alpha_update}\
\State \quad \textbf{Update} $\gamma_{j+1} = (1+\xi\alpha_j)\gamma_{j}$ and $j\leftarrow j+1$
\State \textbf{Output:} $K = {\theta}_{j+1}\widehat{\Phi}^\top$
\end{algorithmic}
\end{algorithm}

Next, we provide the condition on $n_s$, $n_c$, $r$, $\tau$, $T$, $N$, and $\eta$ to guarantee that $K \in \mathcal{S}^1_K$, namely, $K$ is a stabilizing controller for the original system \eqref{eq:LTI_sys mb}.

\section{Sample Complexity Analysis}\label{sec:theoretical_guarantees}

We now present our main results. We first establish the conditions under which the lifted controller $K = \theta_{j+1}\widehat{\Phi}^\top$ stabilizes \eqref{eq:LTI_sys mb}. We then quantify the sample complexity reduction achieved by our approach. To facilitate a clear presentation, we introduce the following key quantities.
\begin{align*}
\varepsilon_\tau:= \sqrt{\frac{J^1_\star}{\mu_{\text{PL}}\left(\du(\ell\log^2\ell)\right)}}, \bar{\lambda}_\theta:= \sqrt{1 - \frac{3(1-\xi)\sigma_{\min}(Q)}{2\bar{J}}}, \bar{J}:=\max\{2J^1_\star, J^{\gamma_0}(0)\}, 
\end{align*}
and $\underline{\alpha}:= {3\sigma_{\min}(Q )}/{\left(2\bar{J} - 3\sigma_{\min}(Q)\right)}$.

\begin{theorem}\label{theorem:main-result} Given $\delta_\tau \in (0,1)$, $\delta_\sigma \in (0,1)$, and $\zeta > 0$. Suppose that $n_s  = \mathcal{O}(\zeta^4\mu^4_\psi\log^6(\ell))\ell$, $n_c = \mathcal{O}(\log(1/\delta_\tau))$, $\tau = \mathcal{O}( \log(1/\varepsilon_\tau)+\tau_0)$, $r = \mathcal{O}(\sqrt{\varepsilon_\tau})$, and $T = \mathcal{O}\left(\log\left( \frac{\ell^{7}(\dx - \ell)\mu_0}{(1 - |\lambda_{\ell+1}|)\varepsilon_{\text{dist}} \delta^3_\sigma} \right)\right)$, with 
$$
\varepsilon_{\text{dist}}:= \min\left\{\frac{\left(1 - \max\{\bar{\lambda}_\theta, |\lambda_{\ell+1}|\}\right)^\ell, }{C_{\text{dist,1}}}, \sqrt{\frac{J^1_\star}{C_{\text{dist,2}}}}\right\}.
$$
In addition, suppose that the number of PG iterations and step-size satisfy
$$
N \geq \frac{\mu_{\text{PL}}}{\eta}\log\left(\frac{2 \bar{J}^2}{(1-\xi)\sigma_{\min}(Q)J^1_\star}\right),\text{ } \eta = \widetilde{\mathcal{O}}\left(1/\left(\du\ell \right)\right).
$$
Then, Algorithm \ref{alg:disc_LQR} returns $K\in \mathcal{S}^1_K$ with $\rho(A+BK) < \bar{\lambda}_\theta$, within $j = \log(1/\gamma_0)/\log(1+\xi\underline{\alpha})$  iterations, with probability $1 - \bar{\delta}$, where $\bar{\delta} := \delta_\sigma + j(\delta_\tau+\bar{c}_1N(\ell^{-\zeta}+n^{-\zeta}_s-n_se^{-\ell/8} - e^{-\bar{c}_2n_s}))$.  
\end{theorem}

Note that $\bar{c}_1$ and $\bar{c}_2$ above are positive constants and the quantities $C_{\text{dist,1}}$ and $ C_{\text{dist,2}}$ are polynomial in the problem parameters $\|A\|$, $\|B\|$, $\|Q\|$, $\nu_\theta$, $L_\theta$, $L_K$ $\phi$, $\ell$ and $\du$. 

Theorem \ref{theorem:main-result} characterizes the convergence of Algorithm \ref{alg:disc_LQR} to a stabilizing controller of \eqref{eq:LTI_sys mb}. In particular, when the learned unstable subspace representation $\widehat{\Phi}$ is sufficiently accurate and, the number of rollouts $n_s$, $n_c$, time horizon $\tau$ and number of iterations $N$ are set large enough, with $r$ and $\eta$ small enough, our algorithm produces a low-dimensional controller that stabilizes the system's unstable dynamics, i.e., $\theta_{j+1} \in \mathcal{S}^1_\theta$. When lifted through $\widehat{\Phi}$, this controller stabilizes \eqref{eq:LTI_sys mb}. Our results also highlight that learning to stabilize on the unstable subspace becomes more demanding (as it requires more data) when the least stable mode, $|\lambda_{\ell+1}|$, approaches marginal stability (i.e., $|\lambda_{\ell+1}| \approx 1$). It is also important to emphasize that, in contrast to \cite{werner2025system}, our work is the first to provide the non-asymptotic guarantees for learning to stabilize LTI systems via the unstable subspace representation with policy gradient. We present the proof of Theorem \ref{theorem:main-result} in Appendix \ref{appendix:sample complexity}. Next, we briefly discuss the idea of the proof. 

\vspace{0.15cm}
\noindent \textbf{Proof idea:}  The first step of the proof is to guarantee that $J^{\gamma_j}(\theta) \leq \bar{J}$ for every iteration. To do so, we use Lemmas \ref{lemma:uniform bounds and lipschitz mb}, \ref{lemma:gradient dominance mb}, and \ref{lemma:zeroth-order estimation mb}) to determine the number of PG iterations $N$ to ensure $J^{\gamma_j}(\theta) \leq \bar{J}$. A preliminary condition on $d(\widehat{\Phi},\Phi)$, and on the estimation parameters $n_s, n_c, \tau$, and $r$, comes from this step, where we set their corresponding error terms to scale as $\mathcal{O}(\bar{J} - J^\gamma_\star)$. We note that such a uniform bound on the cost implies that $\alpha_j$ is uniformly lower bounded as $\alpha_j \geq \underline{\alpha}$. Hence, $\gamma_j$ reaches one within ${\log(1/\gamma_0)}/{\log(1+\xi\underline{\alpha})}$ iterations. Moreover, since $\sqrt{(1+\alpha_j)\gamma_j}\rho(A_u + B_u\theta_{j+1}) < 1$, then it holds that  $\sqrt{(1+\xi\alpha_j)\gamma_j}\rho(A_u + B_u\theta_{j+1}) < \bar{\lambda}_\theta$. We emphasize that $\bar{\lambda}_\theta$ depends on $\xi$, which is set within $(0,1)$ to guarantee that the spectral radius of the closed-loop low-dimensional system is much smaller than one. It then follows from an induction step that $\rho(A_u + B_u \theta_{j+1}) < \bar{\lambda}_\theta$.

The final step of the proof is to demonstrate that $K = \theta_{j+1}\widehat{\Phi}^\top$ stabilizes \eqref{eq:LTI_sys mb}. For this, we note that $A + BK$ is equivalent to  
\begin{align*}
\Omega \left(\begin{bmatrix}
        A_u +B_u\theta_{j+1} \widehat{\Phi}^\top \Phi & B_u\theta_{j+1}\widehat{\Phi}^\top \Phi_\perp\\
        \Delta + B_s \theta_{j+1} \widehat{\Phi}^\top \Phi & A_s + B_s \theta_{j+1} \widehat{\Phi}^\top \Phi_\perp
    \end{bmatrix}\right)\Omega^\top,
\end{align*}
where its spectral radius can be controlled by using the block perturbation bound from \cite{mathias1998quadratic} and the generalized Bauer-Fike theorem \citep{golub2013matrix}. In particular, it is important to remark that the exponential dependence on $\ell$ showing up in $\varepsilon_{\text{dist}}$ follows from the generalized Bauer-Fike theorem, due to the fact that $A_u + B_u\theta_{j+1}$ and $A_s$ are non-diagonalizable.

We are now in place to characterize the sample complexity of Algorithm \ref{alg:disc_LQR}. To do so, we first quantify the sample complexity of the discounted LQR method (i.e., lines 4-8 of Algorithm \ref{alg:disc_LQR}) as the total number of system rollouts, denoted by $\mathcal{S}_c : = j(n_c + n_sN)$ as in \cite{zhao2024convergence}. 

\begin{corollary}\label{corollary:main-result} Let the arguments of Theorem \ref{theorem:main-result} hold. Then, Algorithm \ref{alg:disc_LQR} returns a stabilizing policy for system \eqref{eq:LTI_sys mb} within $\mathcal{S}_c = \log(\rho(A))\widetilde{\mathcal{O}}(\textcolor{blue}{\ell^2} \du)$ trajectories collected from \eqref{eq:LTI_sys mb}. 
\end{corollary}

This result demonstrates the sample complexity reduction achieved by learning to stabilize on the unstable subspace. In contrast to \cite{zhao2024convergence}, where it scales as $\log(\rho(A))\widetilde{\mathcal{O}}(\textcolor{red}{\dx^2} \du)$, our approach significantly improves scalability by requiring a number of rollouts that depends on the number of unstable modes, rather than the full state dimension. It is also important to highlight that Algorithm \ref{alg:disc_LQR} also collects samples from the adjoint system, which scales with $T$ and $\dx$, where the latter is due to the element-wise computations with the adjoint operator. However, we note that for a ``regular" system where the least unstable and stable modes are strictly away from marginal stability, and $\text{gm}(\lambda) = 1$ for the unstable modes, the order of $T$ is negligible. As a result, $\widetilde{\mathcal{O}}(\textcolor{blue}{\ell^2} \du) + \mathcal{O}(\dx)$ is much smaller than $\widetilde{\mathcal{O}}(\textcolor{red}{\dx^2} \du)$ when $\ell \ll \dx$ (i.e., our regime of interest).

\section{Numerical Validation} \label{sec:numerical_results}

We now present numerical experiments to validate and illustrate our theoretical guarantees\footnote{Code is available at \url{https://github.com/jd-anderson/LTS-unstable-representation}.}. Additional experimental results and implementation details are provided in Appendix~\ref{appendix:numerics}.

Consider the cartpole dynamics as our nominal system $(A_0, B_0)$ with four states and single input where $(A_0,B_0)$ are obtained by linearizing (around
the origin) and discretizing with Euler’s method the following equations: 
\begin{align}\label{carpole dynamics}
&(m_c + m_p) \ddot{x} + m_p\ell_p (\ddot{\zeta}\cos(\zeta) - \dot{\zeta}^2\sin(\zeta)) = u, \text{ and } m_p(\ddot{x}\cos(\zeta) + \ell_p\ddot{\zeta} - g\sin(\zeta)) = 0,
\end{align}
\begin{wrapfigure}{r}{0.55\textwidth}
    \centering
    \includegraphics[width=0.55\textwidth]{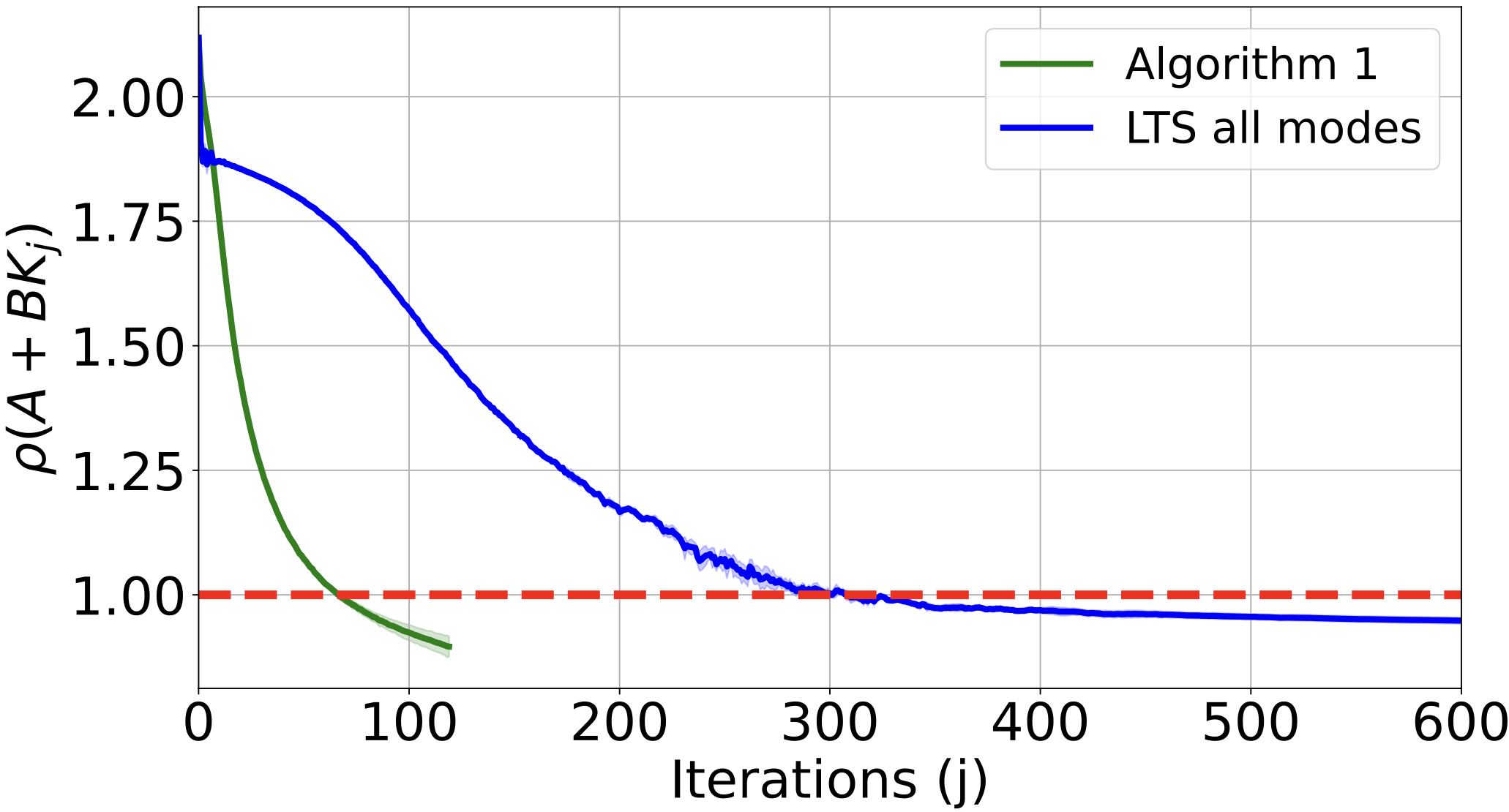}
    \caption{Closed-loop spectral radius w.r.t. iterations.} 
    \label{fig:spectral-radius}
\end{wrapfigure}
where $x$ is the position of the cart and $\zeta$ denotes the angle of the pendulum. In addition, $m_c = 1$, $m_p = 1$, and $\ell_p = 1$ denote the mass of the cart, the mass of the pole, and the length of the pole, respectively. We set the gravitational constant to $g = 10$ and the discretization step-size to $0.25$. The resulting discrete-time LTI dynamics $(A_0,B_0)$ has $\ell = 3$ unstable modes with $|\lambda_1| = |\lambda_2| = 1$ and $|\lambda_3| = 2.12$, where the geometric multiplicity of $\lambda_1$ is equal to one. This nominal system is then augmented by adding random stable modes, resulting in a higher-dimensional system with $\dx = 30$ states and single input, while preserving the original three unstable modes of $(A_0,B_0)$. We adopt $T = 40$ samples from the adjoint system to learn the left unstable subspace representation. Figures \ref{fig:spectral-radius} and \ref{fig:gamma} depict the mean and standard deviation across five runs.

\begin{wrapfigure}{r}{0.55\textwidth}
    \centering
    \includegraphics[width=0.51\textwidth]{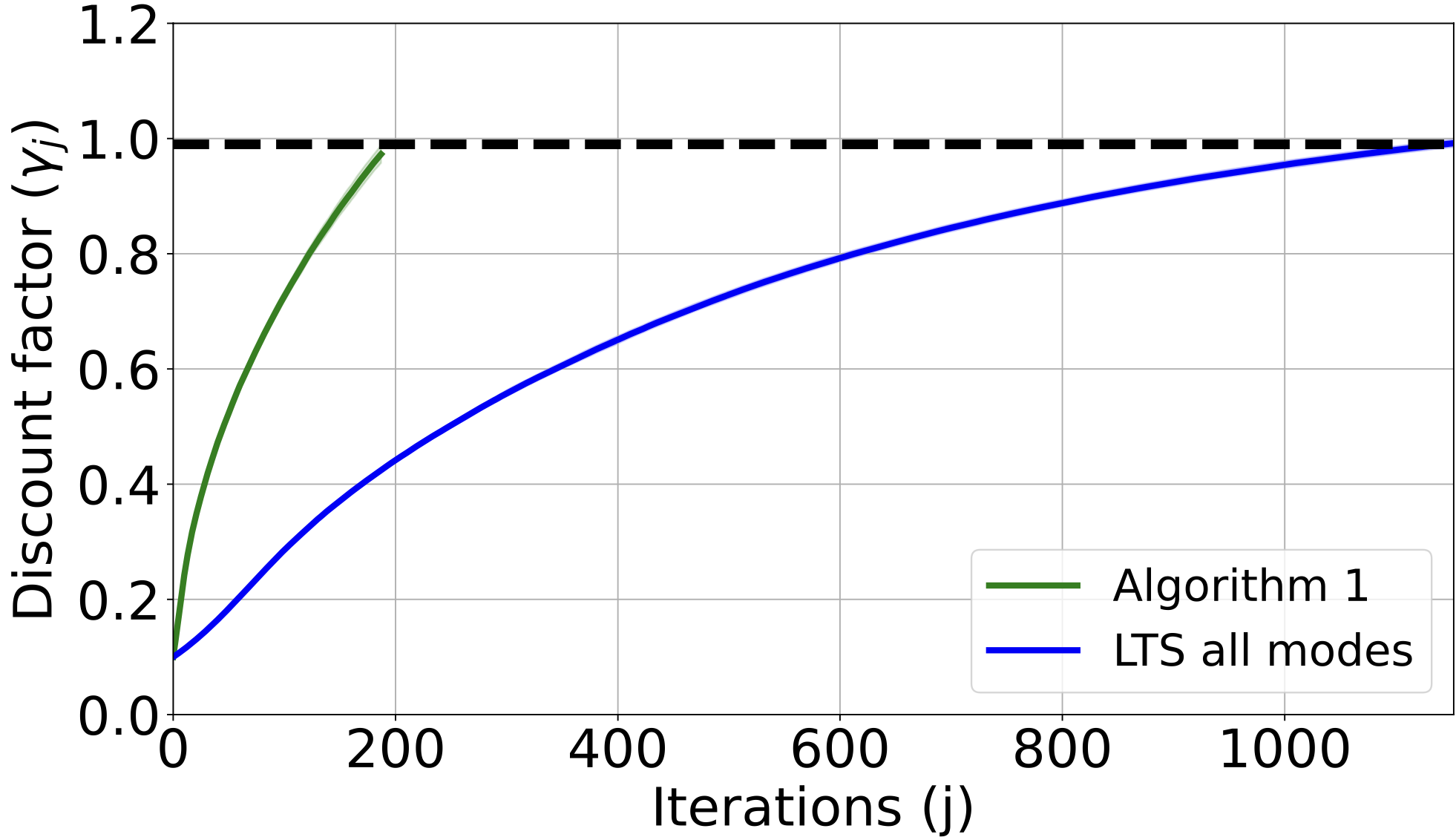}
    \caption{Discount factor w.r.t. iterations.} 
    \label{fig:gamma}
\end{wrapfigure}

Figure \ref{fig:spectral-radius} illustrates the closed-loop spectral radius $\rho(A+BK_{j})$ with respect to the iteration count $j$, for two cases: 1) (green curve) Algorithm~\ref{alg:disc_LQR}, where we learn an unstable subspace representation and perform discounted LQR method to stabilize only the unstable modes the high-dimensional system; 2) (blue curve) applying the discounted LQR method \citep{zhao2024convergence} to stabilize the full dynamics of the high-dimensional system. We note that, by stabilizing only the unstable dynamics while operating on the unstable subspace, Algorithm \ref{alg:disc_LQR} can significantly reduce the number of iterations and thus the amount of samples required to find a stabilizing policy.

This trend is even more pronounced in Figure~\ref{fig:gamma}, which shows the evolution of the discount factor $\gamma_j$ as a function of the iteration count $j$. We observe that Algorithm~\ref{alg:disc_LQR} reaches $\gamma_j = 1$ in approximately 200 iterations, whereas the approach that stabilizes all modes, as in \cite{zhao2024convergence}, requires around 1200 iterations. These results support our theoretical guarantees (i.e., Theorem~\ref{theorem:main-result} and Corollary~\ref{corollary:main-result}) which predict the sample complexity reduction achieved by restricting policy gradient updates to the left unstable subspace in the discounted LQR setting.

\section{Conclusions and Future Work} \label{sec:conclusions}

We studied the problem of learning to stabilize an LTI system. To solve this problem, we proposed a sample efficient method to learn the left unstable space of the system with finite-sample guarantees. We then applied a discount LQR method based on the learned left unstable subspace representation of the system. We proved that when the unstable subspace representation is accurately recovered, the discount method on the unstable subspace returns a stabilizing policy for the original system within a number of iterations that is much smaller than that of learning to stabilize all modes. Compared to existing works, our approach accommodates non-diagonalizable systems and reveal the sample complexity reduction of LTS on the unstable subspace. Future work includes studying the LTS problem for multiple systems with ``similar" unstable subspaces and learning the representation online where we continuously update the learned unstable subspace as more data becomes available.

\section{Acknowledgments} 
The authors thank Bruce D. Lee for instructive discussions at the initial stage of this work. Leonardo F. Toso is funded by the Center for AI and Responsible Financial Innovation (CAIRFI) Fellowship and by the Columbia Presidential Fellowship. James Anderson is partially funded by NSF grants ECCS 2144634 and 2231350. Lintao Ye is supported in part by NSFC grant 62203179.

\bibliographystyle{unsrtnat}
\bibliography{references}

\newpage 
\appendix
\addcontentsline{toc}{section}{Appendix}
\part{Appendix}
\parttoc 
\newpage

\section{Appendix Roadmap}

The appendix is organized as follows. Section~\ref{appendix:numerics} provides additional experiments and further details on the experimental setup used in Section~\ref{sec:numerical_results} to validate our theoretical guarantees. Next, in Section~\ref{appendix:auxiliary results}, we revisit several auxiliary results crucial for deriving the convergence guarantees of Algorithm~\ref{alg:disc_LQR}, including the Davis-Kahan theorem \citep{davis1970rotation} and the generalized Bauer-Fike theorem \citep{golub2013matrix}, which are used to control the subspace distance $d(\widehat{\Phi}, \Phi)$ and the closed-loop spectral radius $\rho(A + B \theta_{j+1} \widehat{\Phi}^\top)$, respectively. We then re-state the discounted LQR problem and the linear decomposition of the high-dimensional control gain $K$ in Sections~\ref{appendix:discounted LQR problem} and~\ref{appendix: linear decomposition of the control policy}, respectively, where we derive an upper bound on $\left\|\nabla J^\gamma(\theta, \Phi) - \nabla J^\gamma(\theta, \widehat{\Phi})\right\|$ in Section~\ref{appendix:gradient error misspecified rep}. The gradient and cost estimation guarantees are presented in Section~\ref{appendix:cost and gradient estimation}. 

Section~\ref{Appendix:learning the left unstable subspace} is dedicated to establishing finite-sample guarantees for learning the left unstable subspace representation, which are then leveraged in Section~\ref{appendix: stabilizing the unstable modes} to derive conditions on the problem parameters under which Algorithm~\ref{alg:disc_LQR} returns a stabilizing policy for system~\eqref{eq:LTI_sys mb}. 

\section{Additional Experiments} \label{appendix:numerics}

For the numerical experiments presented in Section \ref{sec:numerical_results}, we consider the cartpole dynamics \eqref{carpole dynamics} and obtain the following nominal system matrices:
\begin{align*}
    A_0 = \begin{bmatrix}
        1  &  0.25 &  0  &  0\\
        0  &  1  & -2.5 &  0\\
        0  &  0  &  1  &  0.25\\
        0  &  0  & 5  &  1
    \end{bmatrix}, B_0 = \begin{bmatrix}
        0  \\
        0.25\\
        0 \\
        -0.25
    \end{bmatrix}. 
\end{align*}

\noindent \textbf{Augmenting the nominal system:}  $A = \texttt{blkdiag}(A_0, \frac{1}{2}(\tilde{A}+\tilde{A}^\top)/\|\tilde{A}+\tilde{A}^\top\|)$ where $\tilde{A} \in \mathbb{R}^{\dx - 4\times \dx - 4}$ is a random matrix with entries drawn from a normal distribution. In addition, $B =[B_0^\top \;\ \frac{1}{2}\tilde{B}^\top/\|\tilde{B}\| ]^\top$ where $\tilde{B} \in \mathbb{R}^{(\dx - 4) \times \du}$ has also i.i.d. normal distributed entries. 
\vspace{0.2cm}

\noindent \textbf{Problem parameters:} We use $T = 40$ for the number of samples collected from the adjoint system to learn the left unstable subspace representation. The zeroth-order estimation and Algorithm \ref{alg:disc_LQR} parameters are set to $n_s = 20$, $n_c = 100$, $\tau = 50$, $r = 1\times 10^{-3}$, $\gamma_0 = 0.1$, and $\xi = 0.9$. We refer the reader to our code\footnote{Code is available at \url{https://github.com/jd-anderson/LTS-unstable-representation}.} for additional details. 

\vspace{0.2cm}

\noindent \textbf{Inverted Pendulum:} We also provide numerical results for the linearized (around the origin) and discretized (with Euler's method) inverted pendulum dynamics given by 

\begin{align*}
    A_0=\left[\begin{array}{cc}
1 & d t \\
\frac{g}{\ell} d_t & 1
\end{array}\right], B_0=\left[\begin{array}{c}
0 \\
\frac{d_t}{m \ell_p^2}
\end{array}\right],
\end{align*}
where $g = 10$, $m = 1$, $\ell_p = 1$, and $d_t = 0.25$. We augment this nominal system as discussed previously, where the ambient problem dimension is set to $\dx = 20$. The inverted pendulum has a single unstable mode $\lambda_1 = 1.79$ and it is easier to stabilize compared to the cartpole system \eqref{carpole dynamics}. The problem parameters are set as follows: $T = 40$, $n_s = 20$, $n_c = 100$, $\tau = 50$, $r = 1\times 10^{-3}$, $\gamma_0 = 0.1$, and $\xi = 0.9$. Figure \ref{fig:inverted_pendulum} shows the closed-loop spectral radius $\rho(A+BK_{j})$ (left) and the discount factor $\gamma_j$ (right) with respect to the iteration count $j$, for three different cases: 1) (green curve) Algorithm~\ref{alg:disc_LQR} to stabilize a system with $\dx = 20$ states; 2) (blue curve) discounted LQR method, as in \citep{zhao2024convergence}, to stabilize the full dynamics of a system with $\dx = 20$ states; and 3) (purple curve) also learning to stabilize all modes but with a system with $\dx = 10$ states. 

\begin{figure*}
  \centering
  \begin{minipage}{.49\textwidth}
    \centering
    \includegraphics[width=1\textwidth]{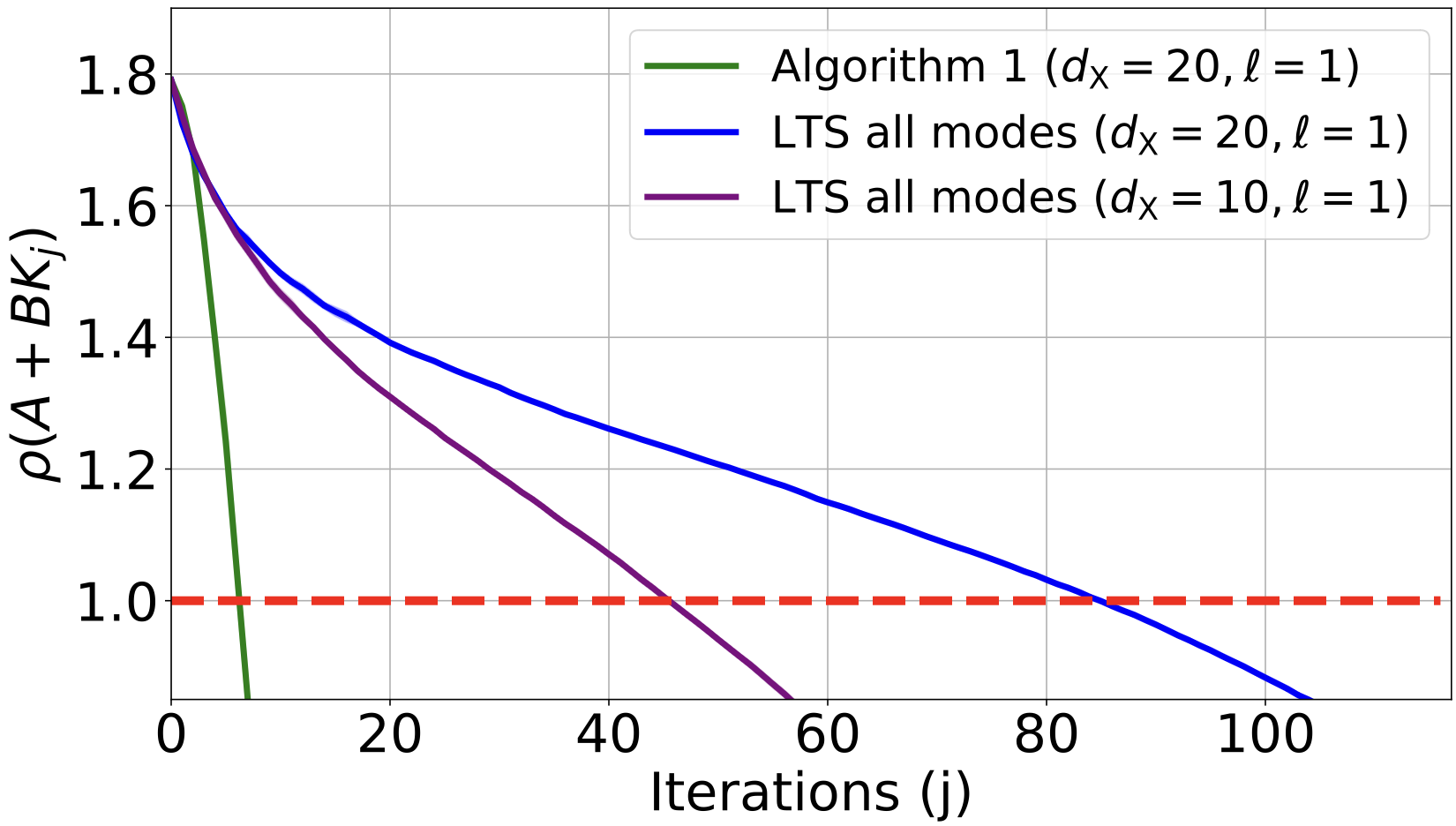}
  \end{minipage}
  \begin{minipage}{.49\textwidth}
    \centering
  \includegraphics[width=1\textwidth]{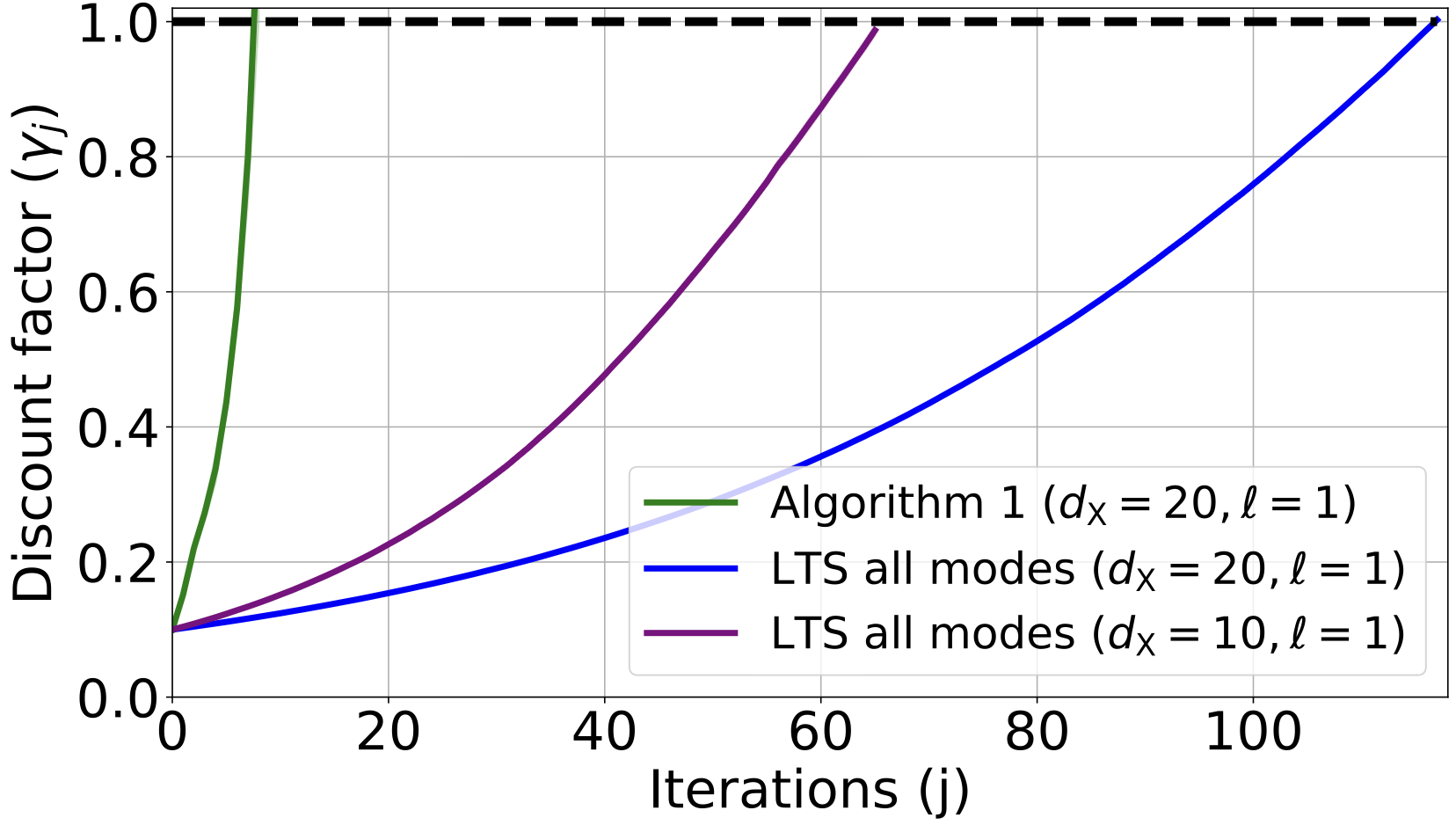}
  \end{minipage}
  \caption{Closed-loop spectral radius (left) and discount factor (right) w.r.t the iteration count.}
  \label{fig:inverted_pendulum}
\end{figure*}

As predicted by our theoretical guarantees (Theorem~\ref{theorem:main-result} and Corollary~\ref{corollary: sample complexity}), Algorithm~\ref{alg:disc_LQR} significantly reduces the number of iterations and thus the number of samples required to find a stabilizing controller. Remarkably, this reduction persists even when compared to LTS all modes of a system with only half the state dimension of the ambient system to which Algorithm~\ref{alg:disc_LQR} is applied.

\vspace{0.2cm}

\noindent \textbf{Random System with Multiple Inputs:} We also provide experimental results for the setting where the system has multiple inputs. In particular, we generate an open-loop unstable system $(A_0,B_0)$ with $A_0 = 2 (\tilde{A}+\tilde{A}^\top)/\|\tilde{A}+\tilde{A}^\top\|$ and $B_0 = \tilde{B}^\top/\|\tilde{B}\|$ with $\tilde{A}$ and $\tilde{B}$ having entries randomly drawn from a normal distribution. We note that $(A_0,B_0)$ is controllable with probability one. In particular, we randomly generate the following system matrices:
\begin{align*}
    A_0 = \begin{bmatrix}
0.68 & 0.68 & -0.16 & 0.49 & 0.45 \\
0.68 & 0.45 & -0.04 & -0.01 & 0.39 \\
-0.16 & -0.04 & 0.62 & 0.77 & 0.47 \\
0.49 & -0.01 & 0.77 & 1.41 & -0.34 \\
0.45 & 0.39 & 0.47 & -0.34 & -0.67
\end{bmatrix}, B_0 = \begin{bmatrix}
-0.20 & 0.21 & 0.19 \\
0.00 & -0.17 & 0.31 \\
0.04 & -0.23 & -0.25 \\
-0.01 & 0.11 & 0.29 \\
0.49 & -0.54 & 0.11
\end{bmatrix},
\end{align*}
which are augmented by following the same procedure discussed previously. 

\begin{figure*}
  \centering
  \begin{minipage}{.49\textwidth}
    \centering
    \includegraphics[width=1\textwidth]{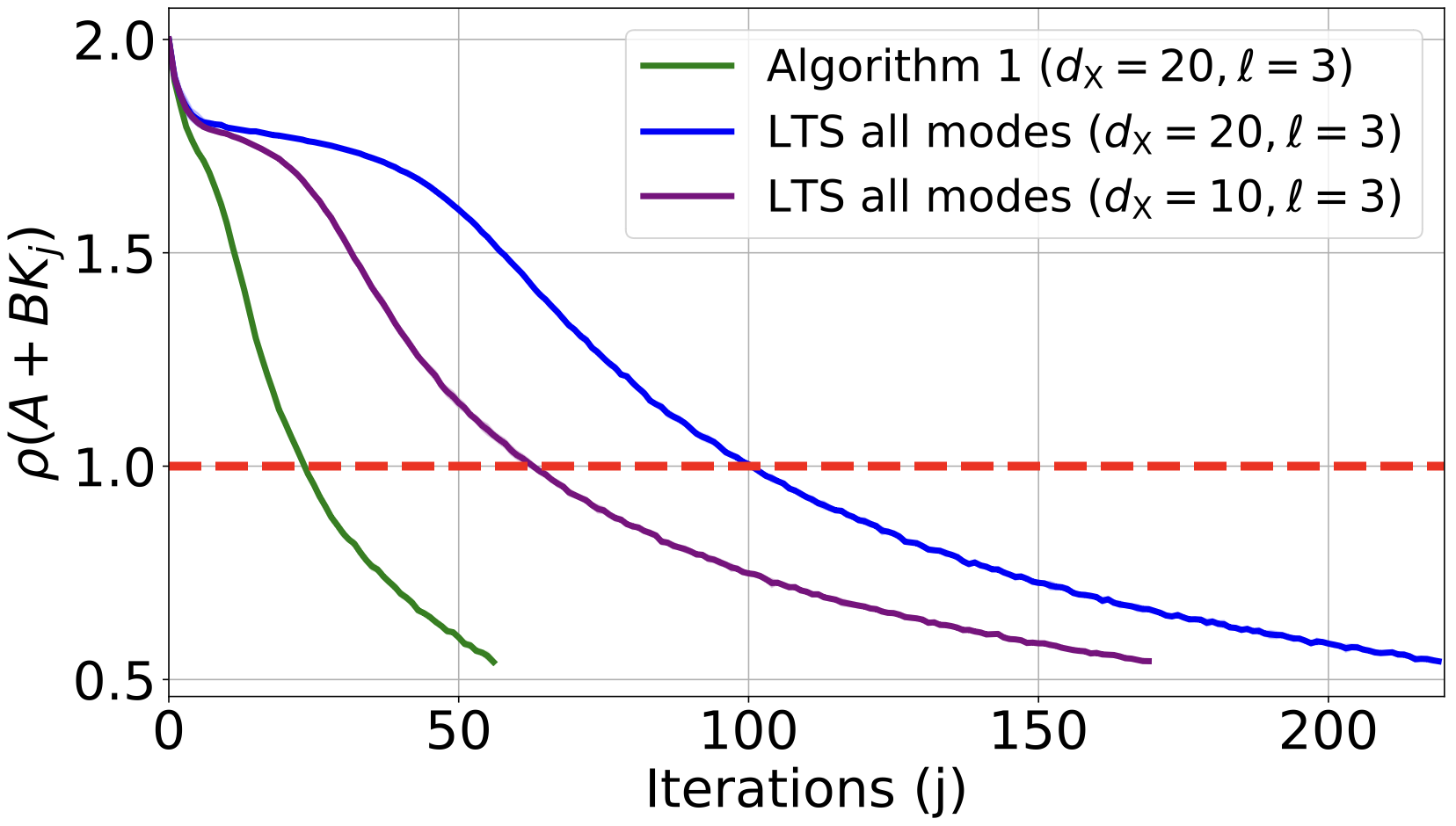}
  \end{minipage}
  \begin{minipage}{.49\textwidth}
    \centering
  \includegraphics[width=1\textwidth]{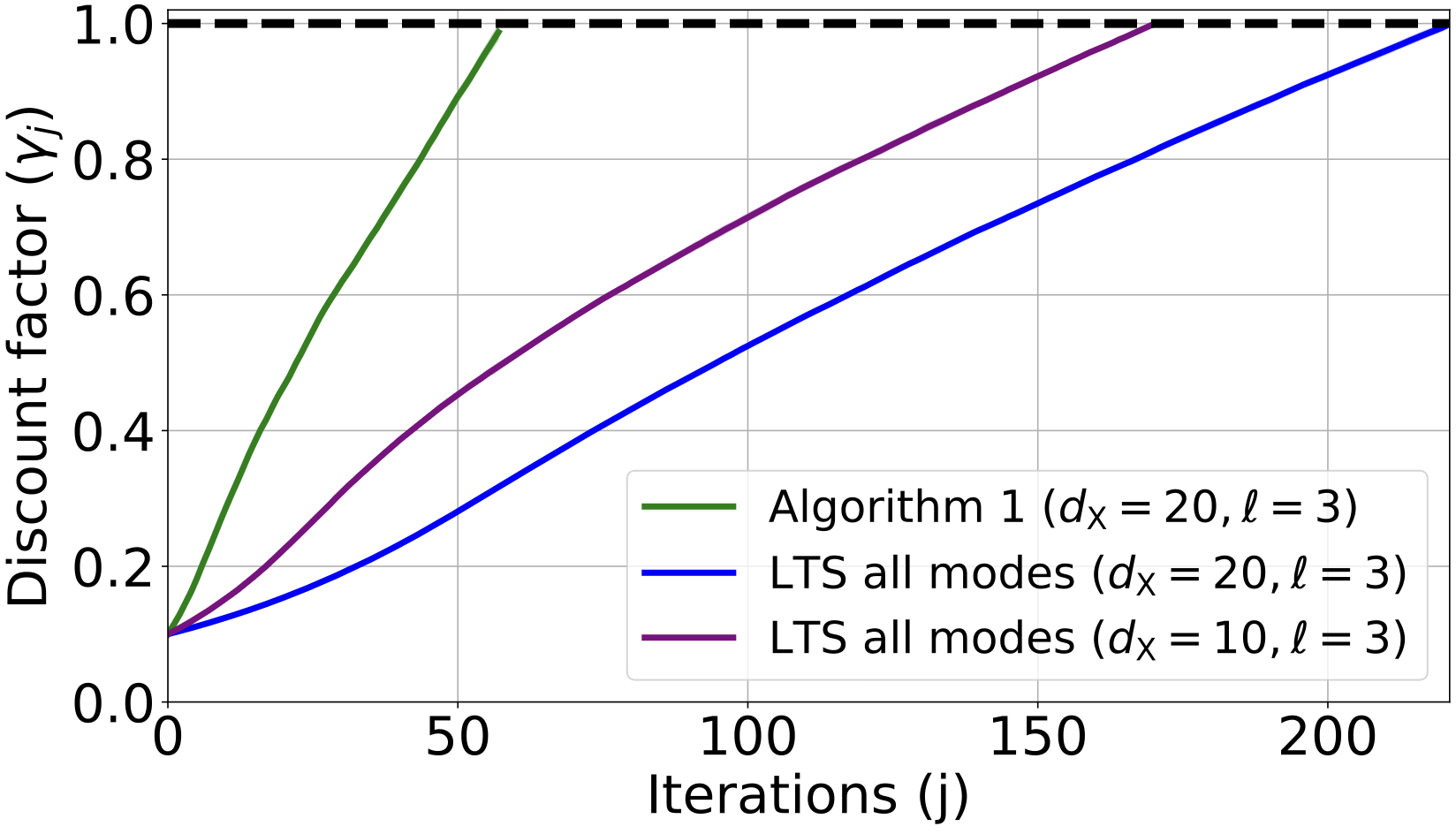}
  \end{minipage}
  \caption{Closed-loop spectral radius (left) and discount factor (right) w.r.t the iteration count.}
  \label{fig:random_matrices}
\end{figure*}

Similar to previous results, Figure~\ref{fig:random_matrices} illustrates the reduction in the number of iterations and overall sample complexity achieved by Algorithm~\ref{alg:disc_LQR}, compared to the approach that stabilizes all modes as in \cite{zhao2024convergence}. These results further validate our theoretical guarantees and highlight the efficiency of the proposed method for learning a stabilizing controller for LTI systems.

\section{Auxiliary Results}\label{appendix:auxiliary results}

\begin{lemma}[Young's inequality]  Given any two real-valued matrices $A, B \in \mathbb{R}^{n\times m}$. It holds that

\begin{align}\label{eq:youngs}
\|A+B\|_2^2 &\leq(1+\beta)\|A\|_2^2+\left(1+\frac{1}{\beta}\right)\|B\|_2^2 \leq (1+\beta)\|A\|_F^2+\left(1+\frac{1}{\beta}\right)\|B\|_F^2,
\end{align}
for any positive scalar $\beta>0$. In addition, we have
\begin{align}\label{eq:youngs_inner_product}
\langle A, B\rangle & \leq \frac{\beta}{2}\lVert A\rVert_2^2 +\frac{1}{2\beta}\lVert B \rVert_2^2  \leq  \frac{\beta}{2}\lVert A\rVert_F^2 +\frac{1}{2\beta}\lVert B \rVert_F^2.
\end{align}
\end{lemma}

\vspace{0.2cm}
\begin{theorem}[\cite{davis1970rotation}]\label{theorem:davis-kahan}
Let $\Sigma$ and $\Sigma + \Delta$ be two $n \times n$ symmetric matrices with spectral decomposition
    \begin{equation*}
        \Sigma = \sum_{j=1}^n \lambda_j u_j u_j^\top, \text{ and }  \Sigma + \Delta = \sum_{j = 1}^n \mu_j v_j v^\top_j,
    \end{equation*}
we also let $\Pi = \sum_{j = 1}^\ell u_j u_j^\top$ and $\Pi^\prime =  \sum_{j = 1}^\ell v_j v_j^\top $ denote the projectors onto the subspace spanned by the top $\ell$ eigenvectors of $\Sigma$ and $\Sigma+\Delta$, respectively. Then, it holds that
\begin{align*}
    \|\Pi  - \Pi^\prime\| \leq \frac{\sqrt{2\ell}\|\Delta
    \|}{\delta},
\end{align*}
where the eigengap  $\delta :=  \inf \left\{|\lambda_i-{\mu_j}|\right\}, \forall i \in \{1,\dots,\ell\}, j \in \{\ell+1,\dots,n\}$.
\end{theorem}
\vspace{0.3cm}
\begin{theorem}[Generalized Bauer-Fike \citep{golub2013matrix}]\label{lemma:bauer-fike} Let $Q^\top A Q=D+N$ be the Schur decomposition of $A \in \mathbb{R}^{d\times d}$, where $D$ is diagonal and $N$ upper triangular with zeros in the diagonal. Then, it holds that
$$
|\rho(A+\Delta) - \rho(A)| \leq \max \left\{\|\Delta\|C_{\text{bf}}, \left(\|\Delta\|C_{\text{bf}}\right)^{1 / d}\right\}, \text{where $C_{\text{bf}} = \sum_{i=0}^{d-1}\|N\|^i$.}
$$
\end{theorem}

\vspace{0.2cm}
\begin{lemma}[Block perturbation bound] \label{lemma:block perturbation}  For any 2-by-2 block matrices $M$ and $N$ in the form
$$
M=\left[\begin{array}{cc}
M_1 &  \\
 & M_2
\end{array}\right], N=\left[\begin{array}{cc}
 & N_{1} \\
N_{2} & 
\end{array}\right],
$$
it holds that $\left|\rho(M+N)-\rho(M)\right| \leq C_{\text{gap}}\left\|N_{1}\right\|\left\|N_{2}\right\|$, where $C_{\text{gap}} = \frac{\kappa(M)\kappa(M+N)}{\min_i\{\text{gap}_i(M)\}}$.
\end{lemma}

In the lemma above, $\kappa(M)$ and $\kappa(M+N)$ denote the condition number of $M$ and $M+N$, respectively. In addition, $\text{gap}_i$ is the (bipartite) spectral gap around $\lambda_i$ with respect to $M$, i.e.,
\begin{align*}
    \operatorname{gap}_i(M):= \begin{cases}\min _{\lambda_j \in \lambda\left(M_2\right)}\left|\lambda_i-\lambda_j\right| & \lambda_i \in \lambda\left(M_1\right) \\ \min _{\lambda_j \in \lambda\left(M_1\right)}\left|\lambda_i-\lambda_j\right| & \lambda_i \in \lambda\left(M_2\right)\end{cases}
\end{align*}
with $\lambda(M_j)$ being the set of eigenvalues of $M_j$ for $j \in \{1,2\}$.

\begin{proof} The proof follows directly from the quadratic residual bounds of non-symmetric matrices from \cite[Theorem 5]{mathias1998quadratic}.
\end{proof}

\section{Discounted LQR Problem} \label{appendix:discounted LQR problem}

We recall that the the discounted LQR problem is defined as follows:

\begin{align}\label{eq:discounted_LQR}
    \text{minimize}_{K \in \mathcal{K}}\left\{J^\gamma (K) := \E\left[ \sum_{t=0}^{\infty} \gamma^t x^\top_t \left(Q  + K^\top RK\right)x_t \right]\right\},  \text{ subject to } \eqref{eq:LTI_sys mb} \text{ with } u_t = Kx_t,
\end{align}
where the expectation is taken with respect to the randomness of the initial state. Moreover, the above discounted LQR problem is equivalent to solve 
\begin{align}\label{eq:LQR_with_damped_sys}
    \text{minimize}_{K \in \mathcal{K}^\gamma} \left\{J^{\gamma}(K) :=  \E\left[ \sum_{t=0}^{\infty} \tilde{x}^\top_t\left(Q + K^\top RK\right)\tilde{x}_t \right]\right\},  \text{ subject to } \tilde{x}_{t+1} = (A^{\gamma}+ B^{\gamma}K)\tilde{x}_t,
\end{align}
where $\tilde{x}_t := \gamma^{t/2}x_t$, $A^{\gamma} := \sqrt{\gamma}A$, $B^{\gamma} := \sqrt{\gamma}B$.
\vspace{0.2cm}
\begin{definition}[Set of stabilizing controllers]Given a discount factor $\gamma \in (0,1]$, the set of stabilizing controllers of the damped system $(A^{\gamma}, B^{\gamma})$ is $\mathcal{K}^\gamma := \{K \mid \sqrt{\gamma}\rho(A + B K) < 1\}$.   
\end{definition}
\vspace{0.2cm}

Given a discount factor $\gamma \in (0,1]$ and stabilizing controller $K \in \mathcal{K}^\gamma$ the discounted LQR cost and its gradient are given by
\begin{align}\label{eq:cost and gradient}
    J^\gamma(K) := \trace{\left(\Sigma^\gamma_K(Q+K^\top R K)\right)},\hspace{0.1cm} \nabla J^\gamma(K) := 2E^\gamma_K \Sigma^\gamma_K,\hspace{0.1cm} \Sigma^\gamma_K := \E\left[\sum_{t=0}^{\infty}x_tx^\top_t\right],
\end{align}
with  $E^\gamma_K := (R+B^{\gamma \top} P^\gamma_K B^{\gamma})K + B^{\gamma \top} P^\gamma_K A^{\gamma} $, where $P^\gamma_K$ is the solution of the closed-loop Lyapunov equation $P^\gamma_K = Q+K^\top R K+ (A^{\gamma}+B^{\gamma} K)^\top P^\gamma_K(A^{\gamma}+B^{\gamma} K).$ Note that the discounted LQR cost can also be written as $J^\gamma (K) = \trace{(P^\gamma_K)}$. 

\section{Linear Decomposition of the Control Policy}\label{appendix: linear decomposition of the control policy}
We consider the linear decomposition of the controller as $K = \theta \Phi^\top$, where $\theta \in \mathbb{R}^{\du\times \ell}$ is a low-dimensional control gain and $\Phi \in \mathbb{R}^{\dx \times \ell}$ is the so-called representation. In addition, $\Phi$ has orthonormal columns. In particular, the columns of $\Phi$ form a basis for the left eigenspace of $A$ corresponding to its unstable modes. Let $z_t \in \mathbb{R}^{\ell}$ be a low-dimensional state that represents $x_t$ in the subspace spanned by the columns of $\Phi$, i.e., $x_t = \Phi z_t$. Therefore, we write
\begin{align}\label{eq:low-dimensional-LTI}
    z_{t+1} = A_uz_t + B_uu_t, \text{ where } A_u = \Phi^\top A \Phi,\text{ } B_u = \Phi^\top B,\text{ and } u_t = \theta z_t,
\end{align}
and state the low-dimensional LQR problem as follows:
\begin{align}\label{eq:cost low dimensional}
\text{minimize}_{\theta \in \Theta} \left\{J^{\gamma}(\theta, \Phi) :=  \E\left[ \sum_{t=0}^{\infty} \gamma^tz^\top_t\left( \Phi^\top Q \Phi + \theta^\top R\theta \right) z_t \right]\right\},  \text{ subject to } z_{t+1} = (A_u + B_u\theta)z_t,
\end{align}
or equivalently 
\begin{align}\label{eq:cost low dimensional 2}
\text{minimize}_{\theta \in \Theta^\gamma} \left\{J^{\gamma}(\theta, \Phi) :=  \E\left[ \sum_{t=0}^{\infty} \tilde{z}^\top_t\left( \Phi^\top Q \Phi + \theta^\top R\theta \right) \tilde{z}_t \right]\right\},  \text{ subject to } \tilde{z}_{t+1} = (A^\gamma_u + B^\gamma_u\theta)\tilde{z}_t,
\end{align}
with $\tilde{z}_t = \gamma^{t/2}z_t$, $A^\gamma_u = \Phi^\top A^\gamma \Phi,\text{ and } B^\gamma_u = \Phi^\top B^\gamma$. In addition, the sets of low-dimensional stabilizing controllers are defined as $\Theta := \{\theta \mid \rho(A_u + B_u \theta) < 1\}$, and $\Theta^\gamma := \{\theta \mid \sqrt{\gamma}\rho(A_u + B_u \theta) < 1\}$. 

Let $\nabla J^\gamma (\theta, \Phi)$ denote the gradient with respect to $\theta$. Therefore, we can write 

\begin{align*}
    \nabla J^\gamma (\theta, \Phi) &= \nabla J^\gamma(\theta \Phi^\top) \Phi = 2\left((R+B^{\gamma \top} P^\gamma_K B^{\gamma})K + B^{\gamma\top} P^\gamma_K A^{\gamma} \right) \E\left[\sum_{t=0}^\infty x_tx^\top_t\right]\Phi\\
    &= 2 \left((R+B^{\gamma \top} P^\gamma_K B^{\gamma})\theta + B^{\gamma\top} P^\gamma_K A^{\gamma}\Phi \right) \E\left[\sum_{t=0}^\infty z_tz^\top_t\right],
\end{align*}
with Lyapunov equation satisfying
\begin{align*}
    \Phi^\top P^\gamma_K \Phi = \Phi^\top Q \Phi + \theta^\top R \theta + \Phi^\top (A^\gamma + B^\gamma \theta \Phi^\top)^\top  P^\gamma_K(A^\gamma + B^\gamma \theta \Phi^\top) \Phi,
\end{align*}
where $P^\gamma_K = \Phi P^\gamma_\theta \Phi^\top$, and thus we have
\begin{align*}
    \nabla J^\gamma (\theta, \Phi) = 2 \left((R+B^{\gamma \top}_u P^\gamma_\theta B^{\gamma}_u)\theta + B^{\gamma\top}_u P^\gamma_\theta A^{\gamma}_u \right) \E\left[\sum_{t=0}^\infty z_tz^\top_t\right],
\end{align*}
with $P^\gamma_\theta = \Phi^\top Q \Phi + \theta^\top R \theta + (A^\gamma_u + B^\gamma_u \theta )^\top  P^\gamma_\theta (A^\gamma_u + B^\gamma_u \theta)$. 

\vspace{0.2cm}

\subsection{Gradient Error Due to Misspecified Representation} \label{appendix:gradient error misspecified rep}

We proceed to control $\left\|\nabla J^\gamma (\theta, \Phi) - \nabla J^\gamma (\theta, \widehat{\Phi})\right\|$, where $\widehat{\Phi}$ is an estimation of $\Phi$. To do so, we have

\begin{align*}
    \left\|\nabla J^\gamma (\theta, \Phi) - \nabla J^\gamma (\theta, \widehat{\Phi})\right\| &= \left\| \nabla J(\theta \widehat{\Phi}^\top)\widehat{\Pi}\widehat{\Phi} - \nabla J(\theta \widehat{\Phi}^\top){\Pi}\widehat{\Phi} + \nabla J(\theta \widehat{\Phi}^\top){\Pi}\widehat{\Phi} - 
    \nabla J(\theta {\Phi}^\top)\Pi \Phi \right\|\\
    &\leq \left\|\nabla J(\theta \widehat{\Phi}^\top)\widehat{\Pi}\widehat{\Phi} - \nabla J(\theta \widehat{\Phi}^\top){\Pi}\widehat{\Phi}   \right\| + \left\|\nabla J(\theta \widehat{\Phi}^\top){\Pi}\widehat{\Phi} - 
    \nabla J(\theta {\Phi}^\top)\Pi \Phi  \right\|\\
    &\leq  \|\nabla J(\theta \widehat{\Phi}^\top)\| \|\widehat{\Pi} - {\Pi}\| + \left\|\nabla J(\theta \widehat{\Phi}^\top){\Pi}\widehat{\Phi} - 
    \nabla J(\theta {\Phi}^\top)\Pi \Phi  \right\|\\
    & \stackrel{(i)}{\leq} \phi d(\widehat{\Phi}, {\Phi}) + \left\|\nabla J(\theta \widehat{\Phi}^\top){\Pi}\widehat{\Phi} - 
    \nabla J(\theta {\Phi}^\top)\Pi \Phi  \right\| \\
    & \leq \phi d(\widehat{\Phi}, {\Phi}) + \left\|\nabla J(\theta \widehat{\Phi}^\top){\Pi}\widehat{\Phi} - \nabla J(\theta {\Phi}^\top){\Pi}\widehat{\Phi}\right\|\\
    &+ \left\| \nabla J(\theta {\Phi}^\top){\Pi}\widehat{\Phi} - 
    \nabla J(\theta {\Phi}^\top)\Pi \Phi  \right\|\\
    &  \stackrel{(ii)}{\leq} \phi d(\widehat{\Phi}, {\Phi}) + L_K\nu_\theta\|\widehat{\Phi} - \Phi\| + \phi\|\widehat{\Phi} - \Phi\|,
\end{align*}
where $(i)$ follows from Lemma \ref{lemma:uniform bounds and lipschitz mb} and Definition \ref{def:subspace_distance}. Moreover, $(ii)$ also follows from Lemma \ref{lemma:uniform bounds and lipschitz mb}. By leveraging \cite[Corollary 5.3]{hu2022sample} we have that $\|\widehat{\Phi} - \Phi\| \leq \sqrt{2\ell}d(\widehat{\Phi},\Phi)$, which implies
\begin{align}\label{eq:error_gradient}
    \left\|\nabla J^\gamma (\theta, \Phi) - \nabla J^\gamma (\theta, \widehat{\Phi})\right\| \leq \left( (L_K\nu_\theta+\phi) \sqrt{2\ell} + \phi\right) d(\widehat{\Phi}, {\Phi}),
\end{align}
or in the Frobenius norm, we have the following:
\begin{align}\label{eq:error_gradient_fro}
    \left\|\nabla J^\gamma (\theta, \Phi) - \nabla J^\gamma (\theta, \widehat{\Phi})\right\|_F \leq \sqrt{\ell}\left( (L_K\nu_\theta+\phi) \sqrt{2\ell} + \phi\right) d(\widehat{\Phi}, {\Phi}).
\end{align}

\section{Gradient and Cost Estimation} \label{appendix:cost and gradient estimation} 

Recall that we operate in model-free, namely, we do not have access to the system matrices $(A,B)$, and thus we cannot directly compute the gradient. Hence, we need to estimate $\nabla J^\gamma (\theta, \widehat{\Phi})$. We proceed by defining the two-point zeroth-order estimation and presenting its guarantees. 

\begin{align*}
    \widehat{\nabla} J^\gamma (\theta, \widehat{\Phi}) := \frac{1}{2rn_s}\sum_{i=1}^{n_s} \left(V^{\gamma,\tau}(\theta_{1,i},z^i_0) - V^{\gamma,\tau}(\theta_{2,i},z^i_0)
  \right)U_i,
\end{align*}
where $U_i$ is drawn from a uniform distribution on the sphere $\sqrt{\ell \du}\mathbb{S}^{\ell\du-1}$. In addition, $\theta_{1,i} = \theta + rU_i$, $\theta_{2,i} = \theta - r U_i$. Note that the initial condition of the low-dimensional system, $z^i_0$, is also distributed according to a zero-mean isotropic distribution, since $\Phi$ has orthonormal columns. Let $\tau > 0$ denote the time horizon. The finite-horizon value function $V^{\gamma, \tau}(\theta, z_0)$ is defined as follows: 
\begin{align*}
    V^{\gamma,\tau}(\theta,z_0) := \sum_{t=0}^{\tau-1} \gamma^t z^\top_t\left(\widehat{\Phi}^\top Q\widehat{\Phi} + \theta^\top R \theta \right)z_t, 
\end{align*}
where $\{z_t\}^{\tau-1}_{t=0} = \{\widehat{\Phi}^\top x_t\}_{t=0}^{\tau-1}$ and $\{x_t\}^{\tau-1}_{t=0}$ is the trajectory data of \eqref{eq:LTI_sys mb} with $u_t = \theta \widehat{\Phi}^\top x_t$. Moreover, let $\widetilde{\nabla}J^\gamma(\theta,\widehat{\Phi}) := \frac{1}{n_s}\sum_{i=1}^{n_s} \left\langle \nabla V^\gamma(\theta,z^i_0), U_i  \right\rangle U_i$ be the unbiased estimate of $\nabla J^\gamma (\theta,\widehat{\Phi})$ where the infinite horizon cost is given by $V^{\gamma}(\theta,z_0) := \sum_{t=0}^{\infty} \gamma^t z^\top_t\left(\widehat{\Phi}^\top Q\widehat{\Phi} + \theta^\top R \theta \right)z_t$. Therefore, it is evident that $\E[\widetilde{\nabla}J^{\gamma}(\theta,\widehat{\Phi}) ] = {\nabla}J^{\gamma}(\theta,\widehat{\Phi})$ \citep[Section IV]{mohammadi2020linear}.
\vspace{0.2cm}
\begin{lemma}[Zeroth-order Estimation Bias \citep{mohammadi2020linear}] \label{lemma: zeroth-order estimation bias} Suppose that $\tau = \mathcal{O}( \log(1/\varepsilon))$ and $r \leq \mathcal{O}(\sqrt{\varepsilon})$. Then, it holds that  $\|\widetilde{\nabla} J^\gamma(\theta,\widehat{\Phi}) 
 - \widehat{\nabla}J^\gamma(\theta,\widehat{\Phi})\|_F \leq  \varepsilon$. 
\end{lemma}
\vspace{0.2cm}

\begin{lemma}[Propositions 3 and 4 of \citep{mohammadi2020linear}]\label{lemma: events} Let $\mu_1$ and $\mu_2$ be two positive scalars, and $\calE_1$ and $\calE_2$ be the following events:
\begin{align*}
    \mathcal{E}_1 &:= \left\{ \left\langle \nabla J^\gamma(\theta,\widehat{\Phi}), \widetilde{\nabla}J^\gamma (\theta,\widehat{\Phi})  \right\rangle   \geq \mu_1 \left\|\nabla J^\gamma (\theta,\widehat{\Phi})\right\|^2_F\right\},
    \mathcal{E}_2 := \left\{ \left\|\widetilde{\nabla} J^\gamma(\theta,\widehat{\Phi})\right\|^2_F   \leq \mu_2 \left\|\nabla J^\gamma (\theta,\widehat{\Phi})\right\|^2_F\right\}.
\end{align*}
Suppose that $n_s =  \mathcal{O}(\zeta^4\mu_\psi^4\log^6(\ell)\ell)$ for some positive scalar $\zeta$. Then, the events $\mathcal{E}_1$ and $\mathcal{E}_2$ hold with probability $1 - c_1(\ell^{-\zeta}+n^{-\zeta}_s-n_se^{-\ell/8} - e^{-c_2n_s})$.
\end{lemma}
\vspace{0.2cm}

In the lemma above, $c_1$ and $c_2$ are positive constants, and the initial condition satisfies $\|z_0\|_{\psi_2} \leq \mu_\psi$. Therefore, by combining Lemmas~\ref{lemma: zeroth-order estimation bias} and \ref{lemma: events}, we obtain the following: 

\begin{align*}
    \|\widetilde{\nabla}J^\gamma (\theta,\widehat{\Phi}) \|^2_F &\leq \mu_2 \|\nabla J^\gamma (\theta,\widehat{\Phi}) - \nabla J^\gamma (\theta,{\Phi}) + \nabla J^\gamma (\theta,{\Phi})\|^2_F\\
    &\stackrel{(i)}{\leq}  2\mu_2 \|\nabla J^\gamma (\theta,\widehat{\Phi}) - \nabla J^\gamma (\theta,{\Phi})\|^2_F + 2\mu_2 \|\nabla J^\gamma(\theta,\Phi) \|^2_F\\
    &\stackrel{(ii)}{\leq} 2\mu_2 \ell\left( (L\nu_\theta+\phi) \sqrt{2\ell} + \phi\right)^2 d(\widehat{\Phi}, {\Phi})^2 + 2\mu_2 \|\nabla J^\gamma(\theta,\Phi) \|^2_F,
\end{align*}
where $(i)$ follows from Young's inequality \eqref{eq:youngs} with $\beta = 1$ and $(ii)$ from \eqref{eq:error_gradient_fro}. We also use \eqref{eq:youngs} to write $\|\widetilde{\nabla}J^\gamma (\theta,\widehat{\Phi})\| \geq -\|\widetilde{\nabla}J^\gamma (\theta,\widehat{\Phi}) - \widehat{\nabla}J^\gamma (\theta,\widehat{\Phi})\|^2_F + \frac{1}{2}\|\widehat{\nabla} J(\theta,\widehat{\Phi})\|^2_F$, which implies that 
\begin{align}\label{eq: bound estimated mispecified grad - norm}
    \|\widehat{\nabla} J^\gamma(\theta,\widehat{\Phi})\|^2_F &\leq 4\mu_2 \|\nabla J^\gamma(\theta,\Phi) \|^2_F +2\|\widetilde{\nabla}J^\gamma (\theta,\widehat{\Phi}) - \widehat{\nabla}J^\gamma (\theta,\widehat{\Phi})\|^2_F\notag \\
    &+ 4\mu_2 \ell\left( (L_K\nu_\theta+\phi) \sqrt{2\ell} + \phi\right)^2 d(\widehat{\Phi}, {\Phi})^2 \notag \\
    &\stackrel{(i)}{\leq} 4\mu_2 \|\nabla J^\gamma(\theta,\Phi) \|^2_F +2\varepsilon^2 + 4\mu_2 \ell\left( (L_K\nu_\theta+\phi) \sqrt{2\ell} + \phi\right)^2 d(\widehat{\Phi}, {\Phi})^2,
\end{align}
where $(i)$ is due to Lemma \ref{lemma: zeroth-order estimation bias}. Similarly, we can write
\begin{align}\label{eq: gradient error inner product 1}
    \langle \nabla J^\gamma(\theta,\widehat{\Phi}), \widetilde{\nabla}J^\gamma(\theta,\widehat{\Phi})  \rangle   &\geq \mu_1 \|\nabla J^\gamma (\theta,\widehat{\Phi})\|^2_F\notag\\
    &\geq  \frac{\mu_1}{2} \left\|\nabla J^\gamma (\theta,{\Phi})\right\|^2_F - \|\nabla J^\gamma (\theta,\widehat{\Phi}) - \nabla J^\gamma (\theta,{\Phi})\|^2_F\notag\\
    &\geq  \frac{\mu_1}{2} \left\|\nabla J^\gamma (\theta,{\Phi})\right\|^2_F-\ell\left( (L\nu_\theta+\phi) \sqrt{2\ell} + \phi\right)^2 d(\widehat{\Phi}, {\Phi})^2,
\end{align}
along with $ T_{\text{grad}}:= \langle \nabla J^\gamma(\theta,\widehat{\Phi}), \widetilde{\nabla}J^\gamma (\theta,\widehat{\Phi})  \rangle $,
\begin{align}\label{eq: gradient error inner product 2}
    T_{\text{grad}} &=  \langle \nabla J^\gamma(\theta,{\Phi}), \widehat{\nabla}J(\theta,\widehat{\Phi})  \rangle + \langle \nabla J^\gamma(\theta,\widehat{\Phi}) - \nabla J^\gamma(\theta,{\Phi}), \widehat{\nabla}J^\gamma (\theta,\widehat{\Phi})  \rangle \notag\\
    &+   \langle \nabla J^\gamma(\theta,\widehat{\Phi}),  \widetilde{\nabla} J^\gamma(\theta,\widehat{\Phi}) - \widehat{\nabla}J^\gamma (\theta,\widehat{\Phi})  \rangle\notag\\
    &\leq \langle \nabla J^\gamma(\theta,{\Phi}), \widehat{\nabla}J^\gamma(\theta,\widehat{\Phi})  \rangle + \frac{\beta}{2}\|\widehat{\nabla} J^\gamma(\theta,\widehat{\Phi})\|^2_F + \frac{1}{2\beta}\|\nabla J^\gamma(\theta,\widehat{\Phi}) - \nabla J^\gamma(\theta,{\Phi})\|^2_F\notag \\
    &+ \frac{\beta}{2}\|\nabla J^\gamma(\theta,\widehat{\Phi}) \|^2_F + \frac{1}{2\beta}\|\widetilde{\nabla} J^\gamma(\theta,\widehat{\Phi}) - \widehat{\nabla}J^\gamma (\theta,\widehat{\Phi})\|^2_F\notag\\
    &\leq  \langle \nabla J^\gamma(\theta,{\Phi}), \widehat{\nabla}J^\gamma(\theta,\widehat{\Phi})  \rangle + \frac{\beta}{2}\|\widehat{\nabla} J^\gamma(\theta,\widehat{\Phi})\|^2_F + \frac{1}{2\beta}\|\nabla J^\gamma(\theta,\widehat{\Phi}) - \nabla J^\gamma(\theta,{\Phi})\|^2_F  \notag \\
    &+ \beta\|\nabla J^\gamma(\theta,{\Phi}) \|^2_F + \beta \|\nabla J^\gamma(\theta,{\Phi}) - \nabla J^\gamma(\theta,\widehat{\Phi})  \|^2_F +  \frac{1}{2\beta}\|\widetilde{\nabla} J^\gamma(\theta,\widehat{\Phi}) - \widehat{\nabla}J^\gamma (\theta,\widehat{\Phi})\|^2_F\notag\\
    & \leq \langle \nabla J^\gamma(\theta,{\Phi}), \widehat{\nabla}J^\gamma(\theta,\widehat{\Phi})  \rangle + \frac{\beta}{2}\|\widehat{\nabla} J^\gamma(\theta,\widehat{\Phi})\|^2_F + \frac{\ell}{2\beta}\left( (L_K\nu_\theta+\phi) \sqrt{2\ell} + \phi\right)^2 d(\widehat{\Phi}, {\Phi})^2  \notag \\
    &+ \beta\|\nabla J^\gamma(\theta,{\Phi}) \|^2_F + \beta \ell\left( (L_K\nu_\theta+\phi) \sqrt{2\ell} + \phi\right)^2 d(\widehat{\Phi}, {\Phi})^2 +  \frac{\varepsilon^2}{2\beta}\notag\\
    &\stackrel{(i)}{\leq} \langle \nabla J^\gamma(\theta,{\Phi}), \widehat{\nabla}J^\gamma(\theta,\widehat{\Phi})  \rangle + \beta(2\mu_2+1) \|\nabla J^\gamma(\theta,\Phi) \|^2_F +\varepsilon^2\left(1+\frac{1}{2\beta}\right)\notag\\
    &+ \left(2\mu_2\beta + \beta + \frac{1}{2\beta}\right) \ell\left( (L_K\nu_\theta+\phi) \sqrt{2\ell} + \phi\right)^2 d(\widehat{\Phi}, {\Phi})^2,
\end{align}
where $(i)$ is due to \eqref{eq: bound estimated mispecified grad - norm}. Hence, with $\beta = \frac{\mu_1}{4(2\mu_2+1)}$ and applying \eqref{eq: gradient error inner product 2} in \eqref{eq: gradient error inner product 1}, we have Lemma \ref{lemma: bounds gradient estimation}.

\vspace{0.2cm}
\begin{lemma} \label{lemma: bounds gradient estimation} Given positive scalars $\mu_1$, $\mu_2$ and $\zeta$. Suppose that we have $n_s =\mathcal{O}( \zeta^4\mu^4_\psi\log^6(\ell)\ell)$, $\tau = \mathcal{O}(\log(1/\varepsilon))$ and $r = \mathcal{O}(\sqrt{\varepsilon})$. Then, it holds that 
 \begin{align*}
     &\|\widehat{\nabla} J(\theta,\widehat{\Phi})\|^2_F \leq 4\mu_2 \|\nabla J^\gamma(\theta,\Phi) \|^2_F  + 4\mu_2 \ell\left( (L_K\nu_\theta+\phi) \sqrt{2\ell} + \phi\right)^2 d(\widehat{\Phi},{\Phi})^2  +2\varepsilon^2,\\
     &\langle \nabla J^\gamma(\theta,{\Phi}), \widehat{\nabla}J^\gamma(\theta,\widehat{\Phi})  \rangle   \geq  \frac{\mu_1}{4} \left\|\nabla J^\gamma (\theta,{\Phi})\right\|^2_F-c_4 \ell\left( (L_K\nu_\theta+\phi) \sqrt{2\ell} + \phi\right)^2 d(\widehat{\Phi}, {\Phi})^2 - c_5 \varepsilon^2,
\end{align*}
 with probability $1 - c_2(\ell^{-\zeta}+n^{-\zeta}_s-n_se^{-\ell/8} - e^{-c_3n_s})$, where $$c_4 = 1 + \frac{2(2\mu_2 + 1)}{\mu_1}, \text{ and } c_5 = \frac{2\mu_2\mu_1}{4(2\mu_2+1)} + \frac{2(2\mu_2+1)}{\mu_1}+ \frac{\mu_1}{4(2\mu_2 + 1)}.$$
\end{lemma}
\vspace{0.2cm}

\subsection{Cost Estimation Error} We conclude this section by revisiting Lemma 5 from \cite{zhao2024convergence}, i.e., the one that controls the error in the cost estimation, and we adapt it to our setting of performing PG on the unstable subspace. Let $\widehat{J}^{\gamma,\tau}(\theta,\widehat{\Phi}) = \frac{1}{n_c}\sum_{i=1}^{n_c} V^{\gamma,\tau}(\theta,z^i_0)$ be the estimated cost with $n_c$ samples and $z^i_0$ denoting a random draw of the low-dimensional initial state. 
\vspace{0.2cm}
\begin{lemma}\label{lemma: cost estimation} Given $\theta \in \mathcal{S}^\gamma_\theta$ and $\delta_\tau \in (0,1)$. Suppose that the time horizon $\tau$, number of rollouts $n_c$, and subspace distance $d(\widehat{\Phi},\Phi)$ satisfy 
\begin{align*}
    \tau \geq \tau_0 := \frac{J^\gamma(\theta,\widehat{\Phi})}{\sigma_{\min}( Q)}\log\left(\frac{8(J^{\gamma}(\theta,\widehat{\Phi}))^2\mu^2_0}{\sigma_{\min}(Q)J^\gamma(\theta)}\right), n_c \geq 8\mu^2_0\log\left(2/\delta_\tau\right), \text{ and } d(\widehat{\Phi},\Phi) \leq J^\gamma(\theta)/(4\sqrt{\ell}C_{\text{cost}}),
\end{align*}
then, it holds that $|\widehat{J}^{\gamma,\tau}(\theta,\widehat{\Phi}) - J^\gamma(\theta)|\leq \frac{1}{2}J^\gamma(\theta)$, with probability $1-\delta_\tau$, where $C_{\text{cost}}$ is polynomial in the problem parameters $\|A\|$, $\|B\|$, $\|Q\|$, $\|R\|$ and $\nu_\theta$.
\end{lemma}

\begin{proof}  The proof follows from first writing
\begin{align*}
    \left|\widehat{J}^{\gamma,\tau}(\theta,\widehat{\Phi}) - J^\gamma(\theta)\right| &= \left|\widehat{J}^{\gamma,\tau}(\theta,\widehat{\Phi}) - J^\gamma(\theta,\widehat{\Phi}) + J^\gamma(\theta,\widehat{\Phi}) - J^\gamma(\theta)\right|\\
    &\leq \left|\widehat{J}^{\gamma,\tau}(\theta,\widehat{\Phi}) - J^\gamma(\theta,\widehat{\Phi})\right| + \left|J^\gamma(\theta,\widehat{\Phi}) - J^\gamma(\theta)\right|,
\end{align*}
where we use \cite[Lemma 5]{zhao2024convergence} to control the first term. Namely, if $\tau$ and $n_c$ are set according to the conditions of Lemma \ref{lemma: cost estimation}, we have that $|\widehat{J}^{\gamma,\tau}(\theta,\widehat{\Phi}) - J^\gamma(\theta,\widehat{\Phi})| \leq \frac{J^\gamma(\theta)}{4}$, with probability $1-\delta_\tau$. On the other hand, for the second term, we have $|J^\gamma(\theta,\widehat{\Phi}) - J^\gamma(\theta)| \leq \ell \|\widehat{P}^\gamma_\theta - P^\gamma_\theta\|$. We then use perturbation bound of the Lyapunov equation, presented in \cite[Proof of Lemma 4]{toso2024meta} to obtain the following:
\begin{align*}
    \|\widehat{P}^\gamma_\theta - P^\gamma_\theta\| \leq C_{\text{cost},1}(\|\widehat{A}^\gamma_u - {A}^\gamma_u\|+\|\widehat{B}^\gamma_u - B^\gamma_u\|) + C_{\text{cost,2}}\|\widehat{\Phi}^\top Q \widehat{\Phi} - {\Phi}^\top Q {\Phi}\|,
\end{align*}
where $\widehat{A}^\gamma_u = \widehat{\Phi}^\top A^\gamma \widehat{\Phi}$ and $\widehat{B}^\gamma_u = \widehat{\Phi}^\top B^\gamma$. In addition, $C_{\text{cost},1}$ and $C_{\text{cost},2}$ are polynomials the problem parameters $\|A\|, \|B\|, \|Q\|, \|R\|, \nu_\theta$. By using \cite[Corollary 5.3]{hu2022sample}, we can write $|J^\gamma(\theta,\widehat{\Phi}) - J^\gamma(\theta)| \leq C_{\text{cost}}\ell \sqrt{\ell}d(\widehat{\Phi},\Phi)$. The proof is completed  by setting $d(\widehat{\Phi},\Phi)\leq \frac{J^\gamma(\theta)}{4\ell\sqrt{l}C_{\text{cost}}}$.

\end{proof}

\section{Learning the Left Unstable Subspace Representation}\label{Appendix:learning the left unstable subspace}

With the data of the adjoint system collected and stored in $D = \left[x_1, x_2, \ldots, x_T   \right] \in \mathbb{R}^{\dx\times T}$, we proceed by taking its singular value decomposition $D = U\Sigma V^\top$, where $U = \left[u_1,u_2,\ldots,u_{\dx}\right] \in \mathbb{R}^{\dx \times \dx}$, $V = \left[v_1,v_2,\ldots,v_{\dx}\right] \in \mathbb{R}^{T \times \dx}$, and $\Sigma = \texttt{diag}(\hat{\sigma}_1,\ldots,\hat{\sigma}_{\dx}) \in \mathbb{R}^{\dx \times \dx}$. The orthonormal basis for the left unstable subspace is constructed with the first $\ell$ columns of $U$, i.e., $\widehat{\Phi} = \left[u_1,\ldots,u_\ell\right]$. Let $\widehat{\Pi} = \widehat{\Phi}\widehat{\Phi}^\top$ and $\Pi = \Phi\Phi^\top$ denote the projectors onto the column spaces of $\widehat{\Phi}$ and $\Phi$, respectively. 

\vspace{0.2cm}
\noindent \textbf{Goal:} Prove that $d(\widehat{\Phi},\Phi) = \|\widehat{\Pi} - \Pi\|$ is sufficiently small when $T$ is sufficiently large. 
\vspace{0.2cm}

To do so, we follow a similar derivation as presented in \cite[Theorem 5.1]{zhang2024learning}, where the main differences in our setting is that we accommodate for non-diagonalizable system matrices $A$, as well as we construct the basis for the left unstable subspace of $A$ rather than the right unstable subspace as in \cite{zhang2024learning}.  Since $A$ is assumed to be real-valued (with potential complex conjugate eigenvalues and eigenvectors) there always exist real basis matrices $\widetilde{\Phi} \in \mathbb{R}^{\dx \times \ell}$ and $\widetilde{\Psi} \in \mathbb{R}^{\dx \times \dx - \ell}$, for the left unstable and stable eigenspaces of $A$, respectively. Hence, we have

\begin{align*}
    A^\top \widetilde{P} = \widetilde{P} \begin{bmatrix} \widetilde{T}_u & 0 \\
                                         0       & \widetilde{T}_s   
                        \end{bmatrix}, \text{ with }\widetilde{P} = \left[\widetilde\Phi \;\   \widetilde\Psi \right] \in \mathbb{R}^{\dx \times \dx}, \widetilde{T}_u \in \mathbb{R}^{\ell \times \ell}, \text{ and }  \widetilde{T}_s \in  \mathbb{R}^{\dx-\ell \times \dx - \ell},
\end{align*}
where $\widetilde{T}_u$ has the same spectrum of the Jordan blocks corresponding to the unstable eigenvalues of $A$, whereas $\widetilde{T}_s$ has the spectrum of the stable counterpart. By orthonormalizing the basis matrices $\widetilde\Phi$ and $\widetilde\Psi$ with a thin QR decomposition we obtain the following:

\begin{align*}
    A^\top \left[\Phi\;\ \Psi\right]  &= \left[\Phi\;\ \Psi\right]\begin{bmatrix} R_{\Phi} & 0 \\
    0 & R_{\Psi}\end{bmatrix} \begin{bmatrix} \widetilde{T}_u & 0 \\
                                         0       & \widetilde{T}_s   
                        \end{bmatrix}\begin{bmatrix} R^{-1}_{\Phi} & 0 \\
    0 & R^{-1}_{\Psi}\end{bmatrix}= \left[\Phi\;\ \Psi\right] \begin{bmatrix} T_u & 0 \\
                                         0       & T_s   
                        \end{bmatrix},
\end{align*}
with $R_{\Phi}$ and $R_{\Psi}$ being the upper triangular matrices for the QR decomposition of $\widetilde\Phi$ and $\widetilde\Psi$, respectively. Their inverses exist due to the fact that $\widetilde\Phi$ and $\widetilde \Psi$ have full column rank. Moreover, we note that $\widetilde{T}_u$ and ${T}_u$ have identical spectrum as well as $\widetilde{T}_s$ and ${T}_s$. Let $\Xi =  \left[\Phi \;\   \Psi \right]$ be composed by the orthonormal basis of the left unstable and stable subspaces of $A$, and denote its inverse by $S = [S_1^\top\;\ S^\top_2 ]^\top := \Xi ^{-1}$. Therefore, we can write
\begin{align*}
    D = \Xi SD = \left[\Phi \;\   \Psi \right]\begin{bmatrix}
        S_1D\\ S_2D
    \end{bmatrix}  = \Phi D_1 + \Psi D_2 = D_u + D_s,
\end{align*}
where $D_1 = S_1D$ and $D_2 = S_2D$. Note that $D$ is composed of $D_u = \Phi D_1$ that comes from the left unstable subspace of $A$, whereas $D_s = \Psi D_2$ depends on the left stable subspace of $A$. The main idea is to collect enough data such that the unstable part dominates the stable one, i.e., the data sufficiently represents the unstable dynamics of \eqref{eq:LTI_sys mb}. 

We proceed our analysis by first considering $D_u$ and writing the singular value decomposition of $D_1$, namely, $D_u = \Phi D_1 = \Phi U_1\Sigma_1 V^\top_1$, with $U_1 \in \mathbb{R}^{\ell \times \ell}$, $\Sigma_1 \in \mathbb{R}^{\ell \times \ell}$, and $V_1 \in \mathbb{R}^{T \times \dx}$. Note that $\widehat{\Pi}$ is the projector onto the subspace spanned by the first $\ell$ columns of $U$, whereas $\Pi$ is the projector onto the columns of $\Phi U_1$. In order to leverage Davis-Kahan theorem (i.e., Theorem \ref{theorem:davis-kahan}) to control the subspace distance, we first write the following symmetric matrices: 
\begin{align*}
    \mathcal{D}_u = \begin{bmatrix}
        0 & D_u^\top\\
        D_u & 0
    \end{bmatrix} = \begin{bmatrix}
        0 & V_1 \Sigma_1 U^\top_1 \Phi^\top\\
        \Phi U_1\Sigma_1 V^\top_1 & 0
    \end{bmatrix}, \mathcal{D}_s = \begin{bmatrix}
        0 & D^\top_s \\
        D_s & 0
    \end{bmatrix}, \mathcal{D} = \mathcal{D}_u + \mathcal{D}_s = \begin{bmatrix}
        0 & D^\top \\
        D & 0
    \end{bmatrix},
\end{align*}
and observe that the eigenvalues and eigenvectors of $\mathcal{D}$ are $\hat{\lambda}_i = \pm \hat{\sigma}_i$ and $[v^\top_i\;\ \pm u^\top_i]^\top$ $\forall i \in [\dx]$. Let $\{\sigma_j\}_{j=1}^{\ell}$ denote the top $\ell$ eigenvalues of $\mathcal{D}_u$ which are the singular values of $D_u$. Therefore, we use Theorem \ref{theorem:davis-kahan} to write 

\begin{align}\label{eq:subpsace_dist}
    d(\widehat{\Phi},\Phi) = \|\widehat{\Pi} - \Pi\| \leq \frac{\sqrt{2\ell}\|D_s\|}{\sigma_l - \hat{\sigma}_{\ell+1}} = \frac{\sqrt{2\ell}\|\Psi D_2\|}{\sigma_l - \hat{\sigma}_{\ell+1}} \leq \frac{\sqrt{2\ell}\|D_2\|}{\sigma_l - \hat{\sigma}_{\ell+1}},
\end{align}
where we control $\|D_2\|$ as follows:
\begin{align}\label{eq:bound_D2}
    \|D_2\| \leq \sqrt{T}\|D_2\|_1 \leq \sqrt{T}\sum_{i = \ell+1}^{\dx} \sum_{t=1}^{T}|\lambda_i|^t\|x_0\| \leq \sqrt{T}\sum_{i = \ell+1}^{\dx} \sum_{t=1}^{\infty}|\lambda_i|^t\mu_0 \stackrel{(i)}{\leq}  \frac{\sqrt{T}(\dx - \ell)\mu_0}{1 - |\lambda_{\ell+1}|},
\end{align}
where $(i)$ is due to the fact that $\{\lambda_i\}_{i=\ell+1}^{\dx}$ are the stable modes of $A$ with $|\lambda_{\ell+1}|\geq \ldots \geq |\lambda_{\dx}|$. Note that the second inequality follows from the fact that $\|D_2\|$ captures the stable dynamics in $\|D\|$, which is in the order of $|\lambda_{i}|^{t}\|x_0\|$ for any stable mode $i \in \{\ell+1,\ldots,\dx\}$ of $A$.  

By combining \eqref{eq:subpsace_dist} and \eqref{eq:bound_D2}, we obtain 
\begin{align}\label{eq:subpsace distance with bound D2}
    d(\widehat{\Phi},\Phi) \leq  \frac{\sqrt{2\ell}\sqrt{T}(\dx - \ell)\mu_0}{(\sigma_\ell - \hat{\sigma}_{\ell+1})(1 - |\lambda_{\ell+1}|)},  
\end{align}
where we now proceed to control $\sigma_l$ and $\hat{\sigma}_{\ell+1}$. First recall that $\sigma_\ell$ is the $\ell$-th top singular value of $D_1$. The following lemma provides a high probability lower bound on $\sigma_{\ell}$.

\begin{lemma}\label{lemma:lower bound on sigmal} Suppose that $T = \mathcal{O}\left(\log(\ell^7/\delta^3_\sigma)/\log(|\lambda_\ell|
)\right)$ for some $\delta_\sigma \in (0,1)$. Then, it holds that 
\begin{align*}
    \sigma_\ell \geq \frac{\sqrt{C_{\sigma}}|\lambda_\ell|^{T}\delta_{\sigma}}{2\sqrt{2}C_{\psi}\ell^{5/2}T^{3/2}}, \text{ with probability $1-4\delta_\sigma$}, \hspace{-0.05cm}\text{ where $C_{\sigma} \hspace{-0.05cm}= \hspace{-0.05cm}\mathcal{O}(1)$.}
\end{align*}
\begin{proof} To prove this lemma, we first define the following quantities:

\begin{align*}
    \phi(A_u,T) := \sqrt{\inf_{v \in S_\ell(1)} \sigma_{\min}\left(\sum_{t = 0}^{T} \Lambda^{-t+1}_u vv^\top\Lambda^{-t+1,\top}_u\right)},
\end{align*}
where $S_\ell(1) := \left\{v \in \mathbb{R}^\ell\mid \min _{1 \leq i \leq \ell} |v_i| \geq 1\right\}$ is the outbox set \citep[Definition 3]{sarkar2019near}, and $A_u = P^{-1}\Lambda_uP$ is the Jordan decomposition of $A_u$, with $P = [P_1\;\ P_2\;\ \ldots \;\ P_\ell]^\top$.
\begin{align*}
    \psi(A_u, T) := \frac{1}{2\ell \sup_{1\leq i \leq \ell} C_{|P^\top_{i}x_0|}},
\end{align*}
where $C_{|P^\top_{i}x_0|}$ is the essential supremum of the pdf of $|P^\top_{i}x_0|$.

\begin{lemma}[\citep{sarkar2019near}] Given $\delta_{\sigma} \in (0,1)$ and suppose that $T$ satisfy 

\begin{align}\label{eq:condition on T}
    4T^2 \sigma_{\max}\left(A^{-(T+1)\varepsilon_T}_u\right)\trace{\left(\Gamma_T(A^{-1}_u)\right)} + \frac{T \trace{\left(A^{-T-1}_u \Gamma_T(A^{-1}_u) (A^{-T-1}_u)^\top\right)}}{\delta_{\sigma}} \leq \frac{\phi^2(A_u,T)\psi^2(A_u,T)\delta^2_{\sigma}}{2},
\end{align}
where we pick $\varepsilon_T$ such that $\varepsilon_T(T+1) = \left\lfloor\frac{T+1}{2}\right\rfloor$, and $\Gamma_T(A_u) = \sum_{t=0}^T A^t_u (A^t_u)^\top$. Then, it holds that $
\sigma_{\ell} \geq \frac{\phi(A_u,T)\psi(A_u,T)\delta_{\sigma}|\lambda_\ell|^{T}}{\sqrt{2}},$ with probability $1-4\delta_{\sigma}$.
\end{lemma}
\vspace{0.2cm}

\noindent $\bullet$ \textbf{Lower bounding $\phi(A_u,T)$:} Let $H(v) = [v \;\ \Lambda_u^{-1} v \;\  \Lambda_u^{-2} v, \ldots, \Lambda_u^{-T+1} v]  = \widetilde{H}\widetilde{V}$, where we define $\widetilde{H} = [I \;\ \Lambda_u^{-1} \;\  \Lambda_u^{-2}, \ldots, \Lambda_u^{-T+1}]$ with $\widetilde{V}$ being an $\ell T\times T$ matrix with $v \in S_{\ell}(1)$ placed accordingly. Therefore, we can write
\begin{align*}
    \phi(A_u,T) = \sqrt{\inf_{v \in S_\ell(1)} \sigma_{\min}\left(H(v)H(v)^\top\right)}, 
\end{align*}
which implies the following:
\begin{align*}
    \phi(A_u,T) = \sqrt{\inf_{v \in S_\ell(1)} \sigma_{\min}\left(H(v)H(v)^\top\right)} = \inf_{v \in S_\ell(1)} \frac{1}{\|H^\dagger\|} \geq \frac{1}{\|\widetilde{H}^\dagger\|(\ell T)^{3/2}},
\end{align*}
with $\|\widetilde{H}^\dagger\| = 1/\sqrt{\sigma_{\min}(\widetilde{H}\widetilde{H}^\top)}$ and $\sigma_{\min}(\widetilde{H}\widetilde{H}^\top) = \sigma_{\min}\left(\sum_{t=0}^{T-1}\Lambda^{-t}_u(\Lambda^{-t}_u)^\top\right) \geq \sum_{t=0}^{T-1}\lambda_{\min}\left(\Lambda^{-t}_u(\Lambda^{-t}_u)^\top  \right)$. 
\begin{align*}
\sigma_{\min}(\widetilde{H}\widetilde{H}^\top) \geq \sum_{t = 0}^{T-1} \frac{1}{\sigma_{\max}(\Lambda_u)^{2t}} \geq \sum_{t = 0}^{T-1}\left(\frac{1}{|\lambda_1|+1}\right)^{2t} := C_\sigma = \mathcal{O}(1),
\end{align*}
and thus $\|\widetilde{H}^\dagger\| \leq 1/\sqrt{C_{\sigma}}$, which yields  $\phi(A_u,T) \geq \frac{\sqrt{C_{\sigma}}}{(\ell T)^{3/2}}$.

\vspace{0.2cm}
Note that to lower bound $\psi(A_u,T)$, we simply need to upper bound $C_{|P^\top_{i}x_0|}$. We recall that $x_0$ is distributed according to a zero-mean and isotropic distribution (e.g., sub-Gaussian distribution), which implies that $C_{|P^\top_{i}x_0|} \leq C_{\psi}$ for some sufficiently large constant $C_{\psi}$. Therefore, we have that $\psi(A_u,T) \geq \frac{1}{2\ell C_{\psi}}$, which can be used to obtain 
$$
\sigma_{\ell} \geq \frac{\phi(A_u,T)\psi(A_u,T)\delta_{\sigma}|\lambda_\ell|^{T}}{\sqrt{2}} \geq \frac{\sqrt{C_{\sigma}}|\lambda_\ell|^{T}}{2\sqrt{2}C_{\psi}\ell^{5/2}T^{3/2}}.
$$

We complete the proof by analyzing the conditions on $T$ to satisfy \eqref{eq:condition on T}. Let us first write
$$
4T^2 \sigma_{\max}\left(A^{-(T+1)\varepsilon_T}_u\right)\trace{\left(\Gamma_T(A^{-1}_u)\right)} + \frac{T \trace{\left(A^{-T-1}_u \Gamma_T(A^{-1}_u) (A^{-T-1}_u)^\top\right)}}{\delta_{\sigma}} \leq \frac{\phi^2(A_u,T)\psi^2(A_u,T)\delta^2_{\sigma}}{2},
$$
which implies that 
$$
4T^3 \ell |\lambda_{\ell}|^{-2(T+1)\varepsilon_T} + \frac{T^2\ell  \sum_{i=1}^{\ell}|\lambda_{i}|^{-2(T+1)}}{\delta_{\sigma}} \leq \frac{\phi^2(A_u,T)\psi^2(A_u,T)\delta^2_{\sigma}}{2},
$$
then we have $4T^3 \ell |\lambda_{\ell}|^{-2(T+1)\varepsilon_T} \leq \frac{\phi^2(A_u,T)\psi^2(A_u,T)\delta^2_{\sigma}}{4}$ and $T^2\ell  \sum_{i=1}^{\ell}|\lambda_{i}|^{-2(T+1)} \leq \frac{\phi^2(A_u,T)\psi^2(A_u,T)\delta^3_{\sigma}}{4}$, which yields 
$T \geq -\frac{\log\left(\frac{\phi^2(A_u,T)\psi^2(A_u,T)\delta^3_{\sigma}}{4\ell^2}\right)}{ \log(|\lambda_{\ell}|)} \geq \frac{\log\left({16 \ell^7 C^2_\psi}/{(C_{\sigma}\delta^3_{\sigma})}\right)}{ \log(|\lambda_{\ell}|)}$ and completes the proof.

\end{proof}
\end{lemma}

Recall that $\hat{\sigma}_{\ell+1}$ denotes the $(\ell+1)$-th singular value of $D$, which corresponds to its stable component $D_s = \Psi D_2$. Consequently, $\hat{\sigma}_{\ell+1}$ is upper bounded by the largest singular value of $D_s$, which in turn is bounded by the largest singular value of $D_2$. This leads to the following:
\begin{align}\label{eq:bound sigma l+1}
    \hat{\sigma}_{\ell+1} \leq \|D_2\| \stackrel{}{\leq} \frac{\sqrt{T}(\dx - \ell)\mu_0}{1 - |\lambda_{\ell+1}|},
\end{align}
where the second inequality follows from \eqref{eq:bound_D2}. By combining Lemma \ref{lemma:lower bound on sigmal} and \eqref{eq:bound sigma l+1} in \eqref{eq:subpsace distance with bound D2}, we have

\begin{align*}
    d(\widehat{\Phi},\Phi) &\leq  \frac{\sqrt{2\ell}\sqrt{T}(\dx - \ell)\mu_0/(1 - |\lambda_{\ell+1}|)}{\frac{\sqrt{C_{\sigma}}|\lambda_\ell|^{T}\delta_{\sigma}}{2\sqrt{2}C_{\psi}\ell^{5/2}T^{3/2}} - \sqrt{T}(\dx - \ell)\mu_0/(1 - |\lambda_{\ell+1}|)} = \frac{4\ell^3{T}^2(\dx - \ell)\mu_0}{(1-|\lambda_{\ell+1}|)\sqrt{C_\sigma}|\lambda_\ell|^{T}\delta_\sigma - 2\sqrt{2}{T}^2\ell^{5/2}(\dx - \ell)\mu_0}\\
    &\stackrel{(i)}{\leq} \frac{8\ell^{3}{T}^2(\dx - \ell)\mu_0}{(1-|\lambda_{\ell+1}|)\sqrt{C_\sigma}|\lambda_l|^{T}\delta_{\sigma} },
\end{align*}
where $(i)$ is due to the selection of $T$ according to $T \geq {\log\left( \frac{4\sqrt{2}\ell^{5/2}(\dx - \ell)\mu_0}{(1 - |\lambda_{\ell+1}|)\sqrt{C_\sigma}\delta_{\sigma}}  \right)}/{\log |\lambda_\ell|}$. We conclude by determining the condition of $T$ that guarantees $d(\widehat{\Phi},\Phi) \leq \varepsilon$, for some small accuracy $\varepsilon$. Namely, 
\begin{align*}
    \frac{8\ell^{3}{T}^2(\dx - \ell)\mu_0}{(1-|\lambda_{\ell+1}|)\sqrt{C_\sigma}|\lambda_l|^{T}\delta_{\sigma} } \leq \varepsilon, \text{ which implies } T \geq {\log\left( \frac{8\ell^{3}(\dx - \ell)\mu_0}{(1 - |\lambda_{\ell+1}|)\sqrt{C_\sigma}\delta_\sigma \varepsilon}   \right) }/{\log |\lambda_{\ell}|},
\end{align*}
with probability $1 - 4\delta_{\sigma}$.

\section{Stabilizing Only the Unstable Modes}\label{appendix: stabilizing the unstable modes}

Given an estimation of the left unstable subspace representation, $\widehat{\Phi}$, we now turn our attention to design a low-dimensional control gain $\theta$ that stabilizes the low-dimensional unstable dynamics $(A_u,B_u)$. To do so, we leverage the explicit discount LQR method from \cite{zhao2024convergence}. Our goal is to guarantee that for every iteration $j$ of the Algorithm \ref{alg:disc_LQR} the cost remains uniformly upper bounded, i.e., $J^{\gamma_j}(\theta_{j}) := J^{\gamma_j}(\theta_{j},\Phi) \leq \bar{J}$, for some positive scalar $\bar{J}$. In addition, the updated discount factor needs to ensure that $\sqrt{\gamma_{j+1}}\rho(A_u+B_u\theta_{j+1}) < 1$ while $\gamma_{j+1} > \gamma_j$.   

\vspace{0.2cm}
\begin{lemma}\label{eq:update disc factor} Given a discount factor $\gamma \in (0,1]$, a decay factor $\xi \in (0,1)$, and a low-dimensional controller $\theta$, such that $\sqrt{\gamma}\rho(A_u+B_u\theta) < 1$. In addition, suppose that $\tau$ and $n_c$ satisfy the conditions of Lemma \ref{lemma: cost estimation}, and suppose  $\gamma_+ = (1+\xi\alpha)\gamma$, with 
\begin{align*}
\alpha = \frac{3\sigma_{\min}\left( \widehat{\Phi}^\top Q \widehat{\Phi} + \theta^\top R\theta \right)}{\frac{4}{3}\widehat{J}^{\gamma,\tau}(\theta,\widehat{\Phi}) - 3\sigma_{\min}\left(\widehat{\Phi}^\top Q \widehat{\Phi} + \theta^\top R \theta \right)},
\end{align*}
then, it holds that $\sqrt{\gamma_+}\rho(A_u+B_u\theta) < 1$. 
\end{lemma}

\begin{proof}  Consider the quadratic Lyapunov function $V(z_t) = z^\top_t P^\gamma_\theta z_t$ for the corresponding low-dimensional damped system $z_{t+1} = \sqrt{\gamma_+}(A_u + B_u\theta)z_t$. Therefore, we can write
\begin{align*}
    V(z_{t+1}) - V(z_t) &= \gamma_+z^\top_t(A_u+B_u\theta)^\top P^\gamma_\theta (A_u+B_u\theta)z_t - z^\top_tP^\gamma_\theta z_t\\
    &\stackrel{(i)}{=} z^\top_t\left(\frac{\gamma_+}{\gamma}\left(P^\gamma_\theta - {\Phi}^\top Q {\Phi} - \theta^\top R \theta\right) - P^\gamma_\theta   \right)z_t,
\end{align*}
where $(i)$ follows from the definition of $P^\gamma_\theta$. Hence, $\frac{\gamma_+}{\gamma}\left(P^\gamma_\theta - \Phi^\top Q \Phi - \theta^\top R \theta\right) - P^\gamma_\theta \prec 0$ ensures that $\sqrt{\gamma_+}\rho(A_u + B_u\theta) < 1$. By applying the trace function on both sides, we have 
$$1 - \frac{\gamma}{\gamma_+} \leq \sigma_{\min}({\Phi}^\top Q {\Phi} + \theta^\top R \theta)/\trace{(P^\gamma_\theta)} \leq \frac{3}{2}\sigma_{\min}(\widehat{\Phi}^\top Q \widehat{\Phi} + \theta^\top R \theta)/\trace{(P^\gamma_\theta)},$$
where the last inequality follows from applying Bauer-Fike theorem \citep{bauer1960norms} along with setting $T$ accordingly to guarantee $d(\widehat{\Phi},\Phi)\leq \frac{\sigma_{\min}(\widehat{\Phi}^\top Q \widehat{\Phi}+\theta^\top R \theta)}{4\|Q\|\sqrt{2\ell}\kappa(\widehat{\Phi}^\top Q \widehat{\Phi}+\theta^\top R \theta)}$. In particular, since the distribution of the initial state is isotropic, we have 
$J^\gamma(\theta,{\Phi}) = \trace{(P^\gamma_\theta)}$, which implies
\begin{align}\label{eq: condition on alpha}
    \gamma_+ &\leq \left(1+ \frac{\frac{3}{2}\sigma_{\min}(\widehat{\Phi}^\top Q \widehat{\Phi} + \theta^\top R \theta)}{J^\gamma(\theta,{\Phi}) - \frac{3}{2}\sigma_{\min}(\widehat{\Phi}^\top Q \widehat{\Phi} + \theta^\top R \theta)}\right)\gamma\notag \\
    &\stackrel{(i)}{\leq} \left(1+ \frac{3\sigma_{\min}(\widehat{\Phi}^\top Q \widehat{\Phi} + \theta^\top R \theta)}{\frac{4}{3}\widehat{J}^{\gamma,\tau}(\theta,\widehat{\Phi}) - 3\sigma_{\min}(\widehat{\Phi}^\top Q \widehat{\Phi} + \theta^\top R \theta)}\right)\gamma = (1 + \alpha)\gamma,
\end{align}
where $(i)$ is due to Lemma \ref{lemma: cost estimation}. As discussed in \cite[Section III]{zhao2024convergence}, the decay factor $\xi \in (0,1)$
is necessary to guarantee that $\sqrt{\gamma_+}\rho(A_u + B_u\theta)$ is strictly away from one. 
\end{proof}
\vspace{0.2cm}

We now proceed to show that for a sufficiently large amount of PG iterations $N$, the LQR cost is uniformly bounded according to $J^{\gamma_j}(\theta_{j+1}) \leq \bar{J}$. Given $\theta_j = \bar{\theta}_0 \in \mathcal{S}^\gamma_\theta$, we use Lemma \ref{lemma:uniform bounds and lipschitz mb} to write 

\begin{align*}
    J^\gamma(\bar{\theta}_{n+1}) - J^\gamma(\bar{\theta}_n) &\leq \langle \nabla J^\gamma(\bar{\theta}_n,\Phi), \bar{\theta}_{n+1} - \bar{\theta}_n\rangle + \frac{L_\theta}{2}\left\|\bar{\theta}_{n+1} - \bar{\theta}_n\right\|^2_F\\
    &\leq -\eta\langle \nabla J^\gamma(\bar{\theta}_n,\Phi), \widehat{\nabla} J^\gamma(\bar{\theta}_n,\widehat{\Phi})\rangle + \frac{L_\theta\eta^2}{2}\left\|\widehat{\nabla} J^\gamma(\bar{\theta}_n,\widehat{\Phi})\right\|^2_F\\
    &\stackrel{(i)}{\leq} -\frac{\eta\mu_1}{4} \left\|\nabla J^\gamma (\bar{\theta}_n,{\Phi})\right\|^2_F+\eta c_4 \ell\left( (L_K\nu_\theta+\phi) \sqrt{2\ell} + \phi\right)^2 d(\widehat{\Phi}, {\Phi})^2 + \eta c_5 \varepsilon^2\\
    &+ \frac{L_\theta\eta^2}{2}\left(4\mu_2 \|\nabla J^\gamma(\bar{\theta}_n,\Phi) \|^2_F  + 4\mu_2 \ell\left( (L_K\nu_\theta+\phi) \sqrt{2\ell} + \phi\right)^2 d(\widehat{\Phi},{\Phi})^2  +2\varepsilon^2\right)\\
    &\stackrel{(ii)}{\leq} -\frac{\eta\mu_1}{8} \left\|\nabla J^\gamma (\bar{\theta}_n,{\Phi})\right\|^2_F+2\eta c_4 \ell\left( (L_K\nu_\theta+\phi) \sqrt{2\ell} + \phi\right)^2 d(\widehat{\Phi}, {\Phi})^2 + 2\eta c_5 \varepsilon^2\\
    &\stackrel{(iii)}{\leq} -\frac{\eta\mu_1}{8\mu_{\text{PL}}}\left(J^\gamma(\bar{\theta}_n) - J^\gamma_\star\right)+2\eta c_4 \ell\left( (L_K\nu_\theta+\phi) \sqrt{2\ell} + \phi\right)^2 d(\widehat{\Phi}, {\Phi})^2 + 2\eta c_5 \varepsilon^2,
\end{align*}
where $J^\gamma_\star = J^\gamma(\theta^\gamma_\star)$, with $\theta^\gamma_\star$ being the optimal controller of the corresponding discounted LQR problem with discount factor $\gamma$. In addition, $(i)$ is due to Lemma \ref{lemma: bounds gradient estimation} and $(ii)$ follows from selecting the step-size according to $\eta \leq \min\left\{ \frac{\mu_1}{16\mu_2L_\theta},\frac{c_4}{2L_\theta \mu_2}, \frac{c_5}{L_\theta}\right\}$. $(iii)$ follows from Lemma \ref{lemma:gradient dominance mb}. Therefore, by adding and subtracting $J^\gamma(\bar{\theta}^\star)$ from both sides, we obtain
\begin{align*}
    J^\gamma(\bar{\theta}_{n+1}) - J^\gamma_\star \leq \left(1 - \frac{\eta\mu_1}{8\mu_{\text{PL}}} \right) \left(J^\gamma(\bar{\theta}_n) - J^\gamma_\star\right) + 2\eta c_4 \ell\left( (L_K\nu_\theta+\phi) \sqrt{2\ell} + \phi\right)^2 d(\widehat{\Phi}, {\Phi})^2 + 2\eta c_5 \varepsilon^2,
\end{align*}
and by unrolling the above expression over $N$ iterations, we have 
\begin{align*}
    J^\gamma(\bar{\theta}_{N}) - J^\gamma_\star \leq \left(1 - \frac{\eta\mu_1}{8\mu_{\text{PL}}}\right)^N\left(J^\gamma(\bar{\theta}_0) - J^\gamma_\star\right) + \frac{16\mu_{\text{PL}}c_4}{\mu_1} \ell\left( (L_K\nu_\theta+\phi) \sqrt{2\ell} + \phi\right)^2 d(\widehat{\Phi}, {\Phi})^2 + \frac{16\mu_{\text{PL}}c_5\varepsilon^2}{\mu_1}, 
\end{align*}
where we can select $\varepsilon$, $d(\widehat{\Phi},\Phi)$, and $N$ according to 
\begin{align}\label{eq: condition on the number of iterations PG}
    \varepsilon \leq \sqrt{\frac{\mu_1(\bar{J} - J^\gamma_\star)}{48\mu_{\text{PL}}c_5}},\text{ } d(\widehat{\Phi},\Phi) \leq \sqrt{\frac{\mu_1(\bar{J} - J^\gamma_\star)}{48\mu_{\text{PL}}c_4l\left( (L_K\nu_\theta+\phi) \sqrt{2\ell} + \phi\right)^2}}, N \geq \frac{8\mu_{\text{PL}}}{\eta\mu_1}\log\left(\frac{3\left(J^\gamma(\bar{\theta}_0) - J^\gamma_\star\right)}{\bar{J} - J^\gamma_\star}\right),
\end{align}
to obtain $J^\gamma(\bar{\theta}_N) = J^\gamma(\theta_{j+1}) \leq \bar{J}$. Therefore, given that $J^{\gamma_j}(\theta_{j+1}) \leq \bar{J}$, for any iteration $j$ of Algorithm \ref{alg:disc_LQR}, and supposing that we select $\tau$ and $n_c$ according to Lemma \ref{lemma: cost estimation} to ensure $\left|\widehat{J}^{\gamma_j, \tau}(\theta_{j+1},\widehat{\Phi}) - J^{\gamma_j}(\theta_{j+1})\right| \leq \frac{1}{2}J^{\gamma_j}(\theta_{j+1})$, we obtain
\begin{align*}
    \alpha_j &= \frac{3\sigma_{\min}(\widehat{\Phi}^\top Q \widehat{\Phi} + \theta^\top R \theta)}{\frac{4}{3}\widehat{J}^{\gamma,\tau}(\theta,\widehat{\Phi}) - 3\sigma_{\min}(\widehat{\Phi}^\top Q \widehat{\Phi} + \theta^\top R \theta)} \geq \frac{3\sigma_{\min}(\widehat{\Phi}^\top Q \widehat{\Phi})}{\frac{4}{3}\widehat{J}^{\gamma,\tau}(\theta,\widehat{\Phi}) - 3\sigma_{\min}(\widehat{\Phi}^\top Q \widehat{\Phi})} \geq \frac{3\sigma_{\min}(Q )}{\frac{4}{3}\widehat{J}^{\gamma,\tau}(\theta,\widehat{\Phi}) - 3\sigma_{\min}(Q)}\\
    &\geq \frac{3\sigma_{\min}(Q )}{2\bar{J} - 3\sigma_{\min}(Q)} := \underline{\alpha}, 
\end{align*}
where we can this lower bound on $\alpha_j$ to unroll the discount factor update over $M$ iterations of Algorithm \ref{alg:disc_LQR} and obtain
\begin{align*}
    \gamma_{M} =  \gamma_0\prod_{j=0}^{M-1} (1+\xi\alpha_j)  \geq \gamma_0\prod_{j=0}^{M-1} (1+\xi\underline{\alpha})  =  \gamma_0 (1+\xi\underline{\alpha})^M,
\end{align*}
which implies that Algorithm \ref{alg:disc_LQR} finds a stabilizing controller $\theta_M \in \mathcal{S}^1_\theta$ (i.e., $\gamma_M = 1$), within $\frac{\log(1/\gamma_0)}{\log(1+\xi\underline{\alpha})}$ iterations. Moreover,  \eqref{eq: condition on alpha} implies that $\sqrt{(1+\alpha_j)\gamma_j}\rho(A_u + B_u\theta_{j+1}) < 1$, which yields
\begin{align*}
     \sqrt{(1+\xi\alpha_j)\gamma_j}\rho(A_u + B_u\theta_{j+1}) &= \frac{\sqrt{(1+\xi \alpha_j)\gamma_j}}{\sqrt{(1+\alpha_j)\gamma_j}}\sqrt{(1+\alpha_j)\gamma_j}\rho(A_u + B_u\theta_{j+1}) < \frac{\sqrt{(1+\xi \alpha_j)\gamma_j}}{\sqrt{(1+\alpha_j)\gamma_j}}\\
     &< \sqrt{1 - \frac{3(1-\xi)\sigma_{\min}(Q)}{2\bar{J}}},
\end{align*}
and it guarantees that after $M$ iterations of the discounted LQR method, it returns a low-dimensional stabilizing controller $\theta \in \mathcal{S}^1_\theta$ with $\rho(A_u+B_u\theta) < \bar{\lambda}_\theta := \sqrt{1 - \frac{3(1-\xi)\sigma_{\min}(Q)}{2\bar{J}}}$. We complete our analysis showing that as long as  $J^{\gamma_j}(\theta_{j+1}) \leq \bar{J}$ and $\left|\widehat{J}^{\gamma_j, \tau}(\theta_{j+1},\widehat{\Phi}) - J^{\gamma_j}(\theta_{j+1})\right| \leq \frac{1}{2}J^{\gamma_j}(\theta_{j+1})$ hold for the $j$-th iteration of Algorithm \ref{alg:disc_LQR}, then they also hold for the subsequent iterations, with high probability. This guarantees that $\alpha_j \geq \underline{\alpha}$ holds for every iteration, and thus $\rho(A_u + B_u\theta_{j+1}) < \bar{\lambda}_\theta$.   

\begin{lemma}\label{eq:next iteration conditions} Suppose that $J^{\gamma_j}(\theta_{j+1}) \leq \bar{J}$ and $\left|\widehat{J}^{\gamma_j, \tau}(\theta_{j+1},\widehat{\Phi}) - J^{\gamma_j}(\theta_{j+1})\right| \leq \frac{1}{2}J^{\gamma_j}(\theta_{j+1})$. Then, it holds that 
\begin{align*}
    \alpha_j \leq \underline{\alpha}, \text{ } J^{\gamma_{j+1}}(\theta_{j+1}) \leq \frac{2 \bar{J}^2}{3(1-\xi)\sigma_{\min}(Q)}.
\end{align*}
\end{lemma}
\begin{proof} The proof is similar to \cite[Lemma 7]{zhao2024convergence} with our definitions of $\underline{\alpha}$ and $\bar{\lambda}_\theta$. 
\end{proof}
\vspace{0.2cm}

Suppose that $\bar{J} > 2J^1_\star$. Then, by the definition we have that $J^{\gamma_0}_\star \leq J^{\gamma_j}(\theta_j)$ which implies that $\bar{J} - J^{\gamma_j}_\star \geq 2J^1_\star - J^1_\star = J^1_\star$. Therefore, according to \eqref{eq: condition on the number of iterations PG} we know that $J^{\gamma_{j}}(\theta_{j+1}) \leq \bar{J}$ if
\begin{align*}
   N \geq \frac{8\mu_{\text{PL}}}{\eta\mu_1}\log\left(\frac{2 \bar{J}^2}{(1-\xi)\sigma_{\min}(Q)J^1_\star}\right), \text{ } \epsilon \leq \sqrt{\frac{\mu_1J^1_\star}{48\mu_{\text{PL}}c_5}}, d(\widehat{\Phi},\Phi) \leq \sqrt{\frac{\mu_1J^1_\star}{48\mu_{\text{PL}}c_4l\left( (L\nu_\theta+\phi) \sqrt{2\ell} + \phi\right)^2}},  
\end{align*}
with probability $1 - (\delta+c_1N(\ell^{-\zeta}+n^{-\zeta}_s-n_se^{-\ell/8} - e^{-c_2n_s}))$. The proof is completed by union bounding over all iterations $j$ of Algorithm \ref{alg:disc_LQR}.

\subsection{Lifting the Controller} 

Given $\theta$ that stabilizes $(A_u, B_u)$, we now demonstrate that $\rho(A+B\theta \widehat{\Phi}^\top) < 1$. To do so, we write 
\begin{align*}
    A + B\theta \widehat{\Phi}^\top &= \Omega \left( \Omega^\top A \Omega + \Omega^\top B \theta \widehat{\Phi} \Omega \right)\Omega^\top = \Omega \left(\begin{bmatrix}
        A_u +B_u\theta \widehat{\Phi}^\top \Phi & B_u\theta\widehat{\Phi}^\top \Phi_\perp\\
        \Delta + B_s \theta \widehat{\Phi}^\top \Phi & A_s + B_s \theta \widehat{\Phi}^\top \Phi_\perp
    \end{bmatrix}\right)\Omega^\top,
\end{align*}
and leverage Lemma \ref{lemma:block perturbation} to obtain

\begin{align}\label{eq:bound closed loop spectral radius}
     \rho(A+B\theta \widehat{\Phi}) & \leq \max\left\{\rho(A_u +B_u\theta \widehat{\Phi}^\top \Phi), \rho(A_s + B_s\theta \widehat{\Phi}^\top \Phi_\perp) \right\} + C_{\text{gap}}\|B_u\theta \widehat{\Phi}^\top \Phi_\perp\| \|\Delta+B_s\theta\widehat{\Phi}^\top \Phi\|\notag \\
     &\leq \max\left\{\rho(A_u +B_u\theta \widehat{\Phi}^\top \Phi), \rho(A_s + B_s\theta \widehat{\Phi}^\top \Phi_\perp) \right\} + C_{\text{gap}}\|B\|\nu_\theta\left(\|A\|+\|B\|\nu_\theta\right) d(\widehat{\Phi},\Phi).
\end{align}

Observe that the second term in the above expression is in the order of the subspace distance. Therefore, we can make it sufficiently small to guarantee that the spectral radius of the closed-loop matrix is less than one. That is a benefit of learning to stabilize on the left unstable subspace instead of the right unstable subspace of $A$. For instance, if the columns of $\Phi$ formed the basis of the right unstable subspace of $A$, the decomposition above would lead to an error term that scales as $\mathcal{O}\left(\|\Delta\| + d(\widehat{\Phi},\Phi)\right)$, where $\|\Delta\|$ is only sufficiently small for $A$ ``almost symmetric" (i.e., where $A$ is easily decomposable into the stable and unstable modes). We now proceed to control the spectral radius: $\rho(A_u +B_u\theta \widehat{\Phi}^\top \Phi)$ and $\rho(A_s + B_s\theta \widehat{\Phi}^\top \Phi_\perp)$. 

\begin{align}\label{eq:bound_sr_Au}
    \rho(A_u +B_u\theta \widehat{\Phi}^\top \Phi) &= \rho(A_u + B_u \theta + B_u\theta (\widehat{\Phi}^\top \Phi-I))\notag \\
    & \stackrel{(i)}{\leq} \rho(A_u + B_u\theta)+\max\left\{\|B_u\theta (\widehat{\Phi}^\top \Phi-I)\|C_{\text{bf},1}, \left(\|B_u\theta (\widehat{\Phi}^\top \Phi-I)\|C_{\text{bf}}\right)^{1/\ell}    \right\} \notag \\
    &\stackrel{(ii)}{\leq}  \rho(A_u + B_u\theta)+
    \left(\|B\|\nu_\theta C_{\text{bf},1}\right)^{1/\ell} d(\widehat{\Phi},\Phi)^{1/\ell}\leq    \bar{\lambda}_\theta +
    \left(\|B\|\nu_\theta C_{\text{bf},1}\right)^{1/\ell} d(\widehat{\Phi},\Phi)^{1/\ell},
\end{align}
where $\bar{\lambda}_\theta :=  \rho(A_u + B_u\theta)$. $(i)$ follows from Lemma \ref{lemma:bauer-fike} with $C_{\text{bf},1}$ being a constant that depends on the Schur decomposition of $A_u+B_u\theta$. $(ii)$ is due to Lemma \ref{lemma:uniform bounds and lipschitz mb} and $d(\widehat{\Phi},\Phi) \leq \frac{1}{\|B\|\nu_\theta C_{\text{bf},1}}$. 

Similarly, we can write 
\begin{align}\label{bound_sr_As}
    \rho(A_s +B_s\theta \widehat{\Phi}^\top \Phi_\perp)  \leq  |\lambda_{\ell+1}|+
    \left(\|B\|\nu_\theta C_{\text{bf},2}\right)^{1/\ell} d(\widehat{\Phi},\Phi)^{1/\ell},
\end{align}
where $C_{\text{bf},2}$ depends on the Schur decomposition of $A_s$. In addition, we require the subspace distance to satisfy $d(\widehat{\Phi},\Phi) \leq \frac{1}{\|B\|\nu_\theta C_{\text{bf},2}}$. By combining \eqref{eq:bound_sr_Au} and \eqref{bound_sr_As} in \eqref{eq:bound closed loop spectral radius} to obtain 
\begin{align*}
    \rho(A+B\theta\widehat{\Phi}^\top) \leq \max\{\bar{\lambda}_\theta,\lambda_{\ell+1}\} + \left(C_{\text{gap}}\|B\|\nu_\theta\left(\|A\|+\|B\|\nu_\theta\right) + \left(\|B\|\nu_\theta\right)^{1/\ell}\left(C^{1/\ell}_{\text{bf},1} + C^{1/\ell}_{\text{bf},2}\right) \right)d(\widehat{\Phi},\Phi)^{1/\ell},
\end{align*}
which implies that 
\begin{align*}
    d(\widehat{\Phi},\Phi) < \frac{\left(1 - \max\{\bar{\lambda}_\theta, |\lambda_{\ell+1}|\}\right)^\ell}{\left(C_{\text{gap}}\|B\|\nu_\theta\left(\|A\|+\|B\|\nu_\theta\right) + \left(\|B\|\nu_\theta\right)^{1/\ell}\left(C^{1/\ell}_{\text{bf},1} + C^{1/\ell}_{\text{bf},2}\right) \right)^\ell},
\end{align*}
to guarantee that $\rho(A+B\theta\widehat{\Phi}^\top) < 1$.

\vspace{0.2cm}
\begin{theorem}[Main Result] \label{theorem:main result} Given positive scalars $\delta_\tau \in (0,1)$, $\delta_\sigma \in (0,1)$ and $\zeta > 0$. Suppose that the problem parameters are selected as follows:   

\vspace{0.2cm}

\noindent $\bullet$ \underline{Gradient and cost estimation parameters:}

    \begin{align*}
        &n_s = \mathcal{O}\left(\ell\zeta^4\log^6\ell \right)\text{, } n_c= \mathcal{O}\left(\log\left(1/\delta_\tau\right)\right),  \varepsilon^\prime:= \sqrt{\frac{J^1_\star}{\mu_{\text{PL}}\left(\du(\ell\log^2\ell)\right)}}\text{, }\\
        &r = \mathcal{O}\left(\sqrt{\varepsilon^\prime}\right)\text{, and } \tau = \mathcal{O}\left( \log(1/\varepsilon^\prime) +\tau_0  \right). \\
    \end{align*}

\vspace{0.2cm}
\noindent $\bullet$ \underline{Subspace distance:} $d(\widehat{\Phi},\Phi) \leq \varepsilon_{\text{dist}}$ with 
\begin{small}
 \begin{align*}
        \varepsilon_{\text{dist}} :=  \min\left\{\frac{\left(1 - \max\{\bar{\lambda}_\theta, |\lambda_{\ell+1}|\}\right)^\ell, }{C_{\text{dist,1}}}, \sqrt{\frac{J^1_\star}{C_{\text{dist,2}}}},\frac{1}{\|B\|\nu_\theta \max\{C_{\text{bf},1},C_{\text{bf},1}\}},\frac{J^1_\star}{4\ell\sqrt{l}C_{\text{cost}}},\frac{\sigma_{\min}(\widehat{\Phi}^\top Q \widehat{\Phi})}{4\|Q\|\sqrt{2\ell}\kappa(\widehat{\Phi}^\top Q \widehat{\Phi})} \right\},
    \end{align*}   
\end{small}
and $C_{\text{dist,1}} = \texttt{poly}\left(\|A\|, \|B\|, \nu_\theta\right) \text{and } C_{\text{dist,2}} = \texttt{poly}(\nu_\theta, L, \mu_{\text{PL}}, \phi, \ell, \du )$. 

\vspace{0.2cm}
\noindent $\bullet$ \underline{Time horizon:} 

    \begin{align*}
        T  = \mathcal{O}\left(\log\left( \frac{\ell^{7}(\dx - \ell)\mu_0}{(1 - |\lambda_{\ell+1}|)\varepsilon_{\text{dist}} \delta^3_\sigma} \right)/\log(|\lambda_{\ell}|)\right).
    \end{align*}

\vspace{0.2cm}
\noindent $\bullet$  \underline{Algorithm \ref{alg:disc_LQR} parameters:}

    \begin{align*}
        N \geq \frac{32\mu_{\text{PL}}}{\eta}\log\left(\frac{2 \bar{J}^2}{(1-\xi)\sigma_{\min}(Q)J^1_\star}\right),\text{ } \eta = \mathcal{O}\left(1/\left(\du\ell \log^2\ell\right)\right),\text{ } \gamma_0 \leq {1}/{\rho^2(A)},\text{ and } \xi \in (0,1), 
    \end{align*} 
\noindent with $\bar{J}:=\max\{2J^1_\star, J^{\gamma_0}(0)\}$. Therefore, within $M \geq \frac{\log(1/\gamma_0)}{\log(1+\xi\underline{\alpha})}$ iterations, Algorithm \ref{alg:disc_LQR} returns $K = \theta_M\widehat{\Phi}^\top \in \mathcal{K}$, with probability $1 - \bar{\delta}$, where $\bar{\delta} := \delta_\sigma + M(\delta_\tau+\bar{c}_1N(\ell^{-\zeta}+n^{-\zeta}_s-n_se^{-\ell/8} - e^{-\bar{c}_2n_s}))$. 
\end{theorem}

\vspace{0.2cm}
\subsection{Policy Gradient Per-iteration Stability Analysis} \label{appendix:stability guarantees}

Given a discount factor $\gamma \in (0,1]$, we proceed to demonstrate that $\bar{\theta}_{n} \in \mathcal{S}^{\gamma}_\theta$, for any policy gradient update $n \in \{0,1,\ldots,N-1\}$ of Algorithm \ref{alg:disc_LQR} (i.e., line 6). As discussed previously, by carefully incrementing $\gamma$, the low-dimensional control gain $\theta_{+} = \bar{\theta}_N$ is stabilizing for the underlying low-dimensional damped system \eqref{eq:low-dimensional-LTI mb}. Therefore, it remains to show that for any $\bar{\theta}_0 \in \mathcal{S}^{\gamma}_\theta$, $\bar{\theta}_n$ stays within $\mathcal{S}^{\gamma}_\theta$. To do so, we can first show that $\bar{\theta}_1$ does not leave $\mathcal{S}^\gamma_\theta$, with high probability, as long as the problem parameters are set accordingly. Finally, we use an induction step to extend this conclusion for any iteration $n \in \{0,1,\ldots,N-1\}$.

As previously, we use Lemma \ref{lemma:uniform bounds and lipschitz mb} to write 

\begin{align*}
    J^\gamma(\bar{\theta}_{1}) - J^\gamma_{\star} \leq \left(1-\frac{\eta\mu_1}{8\mu_{\text{PL}}}\right)\left(J^\gamma(\bar{\theta}_0) - J^\gamma_\star\right)+2\eta c_4 \ell\left( (L_K\nu_\theta+\phi) \sqrt{2\ell} + \phi\right)^2 d(\widehat{\Phi}, {\Phi})^2 + 2\eta c_5 \varepsilon^2,
\end{align*}
where the step-size is set according to $\eta \leq \min\left\{ \frac{\mu_1}{16\mu_2L_\theta},\frac{c_4}{2L_\theta \mu_2}, \frac{c_5}{L_\theta}\right\}$. Therefore, as long as the subspace distance $d(\widehat{\Phi}, {\Phi})$, time horizon $\tau$ and smoothing radius $r$ are set according to
\begin{align*}
    d(\widehat{\Phi}, {\Phi}) \leq \frac{\sqrt{\mu_1}}{8\left( (L_K\nu_\theta+\phi) \sqrt{2\ell} + \phi\right) \sqrt{c_4 \ell \mu_{\text{PL}}}}, \tau = \mathcal{O}(\log(1/\varepsilon)), \text{ and } r = \mathcal{O}(\sqrt{\varepsilon}), \text{ respectively,}
\end{align*}
with $\varepsilon \leq \frac{1}{8}\sqrt{{\mu_1}/{c_5\mu_{\text{PL}}}}$. Hence, we obtain $J^\gamma(\bar{\theta}_{1}) - J^\gamma_{\star} \leq \left(1-\frac{\eta\mu_1}{16\mu_{\text{PL}}}\right)\left(J^\gamma(\bar{\theta}_0) - J^\gamma_\star\right)$, which implies that $\bar{\theta}_1$ is stabilizing for the underlying damped system, i.e., $\bar{\theta}_1 \in \mathcal{S}^\gamma_{\theta}$. Let the base case and inductive hypothesis be defined as follows:

\begin{align*}
    &\textbf{Base case:} \;\ J^\gamma(\bar{\theta}_{1}) - J^\gamma_{\star} \leq J^\gamma(\bar{\theta}_0) - J^\gamma_\star,\\
    &\textbf{Inductive hypothesis:} \vspace{0.3cm}\;\ J^\gamma(\bar{\theta}_{n}) - J^\gamma_{\star} \leq J^\gamma(\bar{\theta}_0) - J^\gamma_\star,
\end{align*}
which combined with the aforementioned conditions on the problem parameters yields
\begin{align*}
    J^\gamma(\bar{\theta}_{n+1}) - J^\gamma_{\star} &\leq \left(1-\frac{\eta\mu_1}{8\mu_{\text{PL}}}\right)\left(J^\gamma(\bar{\theta}_n) - J^\gamma_\star\right)+2\eta c_4 \ell\left( (L_K\nu_\theta+\phi) \sqrt{2\ell} + \phi\right)^2 d(\widehat{\Phi}, {\Phi})^2 + 2\eta c_5 \varepsilon^2\\
     &\leq \left(1-\frac{\eta\mu_1}{16\mu_{\text{PL}}}\right)\left(J^\gamma(\bar{\theta}_n) - J^\gamma_\star\right) \leq \left(1-\frac{\eta\mu_1}{16\mu_{\text{PL}}}\right)\left(J^\gamma(\bar{\theta}_0) - J^\gamma_\star\right) \leq J^\gamma(\bar{\theta}_0) - J^\gamma_\star.
\end{align*}

\vspace{0.2cm}
\subsection{Sample Complexity Reduction} \label{appendix:sample complexity}

We proceed to characterize the sample complexity of Algorithm \ref{alg:disc_LQR} and the benefit of learning to stabilize on the unstable subspace. We quantify the sample complexity by the number of data samples $x_t$ we query from the system \eqref{eq:LTI_sys mb} and its adjoint. Namely, $\mathcal{S}_c := \mathcal{S}^1_c + \mathcal{S}^2_c$, where $\mathcal{S}^1_c := M(n_c+n_sN)\tau$ includes the samples used discounted LQR method to learn a low-dimensional control gain that stabilizes the unstable dynamics, and $\mathcal{S}^2_c := T + \dx$ corresponds to the number of data points needed for estimating the left unstable subspace of $A$. We emphasize that the extra $\dx$ term comes from sampling data from the adjoint system through element-wise computations via the adjoint operator, as discussed previously in Section \ref{sec: learning the left unstable reprensentation}.

\vspace{0.2cm}
\begin{corollary}\label{corollary: sample complexity} Let the arguments of Theorem \ref{theorem:main result} hold. Then, Algorithm \ref{alg:disc_LQR} returns a stabilizing controller for the original system \eqref{eq:LTI_sys mb} with 

$$\mathcal{S}_c = \log(\rho(A))\widetilde{\mathcal{O}}(\textcolor{blue}{\ell^2} \du)C_{\text{sc,1}} + \mathcal{O}\left(\log \left( \frac{\ell^{7}(\dx - \ell)C_{\text{sc,2}}}{(1 - |\lambda_{\ell+1}|)\left(1 - \max\{\bar{\lambda}_\theta, |\lambda_{\ell+1}|\}\right)^\ell } \right)\right) + \mathcal{O}(\dx),$$
where $C_{\text{sc,1}} = \texttt{poly}(\|A\|,\|B\|,\|Q\|, \mu_{\text{PL}})$ and $C_{\text{sc,2}} = \texttt{poly}\left(\|A\|, \|B\|, \nu_\theta, L, \mu_{\text{PL}}, \phi, \ell, \du, 1/J^1_\star,1/\delta_{\sigma}\right)$.
\end{corollary}
\vspace{0.2cm}

Note that the sample complexity is dominated by $\widetilde{\mathcal{O}}(\textcolor{blue}{\ell^2}\du) + \mathcal{O}(\dx)$ which scales much slower than $\widetilde{\mathcal{O}}(\textcolor{red}{\dx^2}\du)$ for the setting where the number of unstable modes is much smaller than the number of states of the system, i.e., our setting of interest with $\ell \ll \dx$.

\end{document}